\pdfoutput=1
\documentclass[11pt]{article}
\pdfoutput=1
\usepackage[margin=0.95in]{geometry} 
\usepackage{kpfonts}

\usepackage[colorlinks=true,citecolor=blue]{hyperref}
\usepackage{url}

\usepackage{lipsum}



\usepackage[utf8]{inputenc} 
\usepackage[T1]{fontenc}    
\usepackage{booktabs}       
\usepackage{amsfonts}       
\usepackage{nicefrac}       
\usepackage{microtype,soul}      
\usepackage{natbib}

\usepackage{amsfonts}
\usepackage{xcolor}
\let\proof\relax
\let\endproof\relax
\usepackage{amsthm}
\usepackage{amsmath}
\usepackage{amssymb}
\usepackage{soul}
\usepackage{comment}
\usepackage{enumitem}
\usepackage{bbm}
\usepackage{graphicx}
\usepackage{float}
\usepackage{subfigure}
\usepackage{mathtools}
\usepackage[ruled,linesnumbered,noend,noline]{algorithm2e}
\usepackage{bm}
\usepackage{accents}

\newtheorem{definition}{Definition}

\newtheorem{theorem}{Theorem}
\newtheorem{corollary}{Corollary}
\newtheorem{lemma}{Lemma}

\newtheorem{remark}{Remark}

\theoremstyle{remark}

\newcommand{\ra}{\rightarrow}
\newcommand{\mb}[1]{\mathbb{#1}}
\newcommand{\mbf}[1]{\mathbf{#1}}
\newcommand{\mc}[1]{\mathcal{#1}}

\newcommand{\rr}{\mathbb{R}}
\newcommand{\ee}{\mathbb{E}}
\newcommand{\pp}{\mathbb{P}}

\renewcommand{\epsilon}{\varepsilon}
\renewcommand{\phi}{\varphi}

\newcommand{\inner}[1]{\left\langle #1 \right\rangle}

\renewcommand{\l}{\left(}
\renewcommand{\r}{\right)}

\newcommand{\indicator}{\mathbbm{1}}
\newcommand{\grad}{\nabla}

\newcommand{\ol}{\overline}
\renewcommand{\ul}{\underline}
\newcommand{\ut}{\undertilde}

\newcommand\defeq{\stackrel{\mathclap{\normalfont\tiny\mbox{def}}}{=}}

\usepackage{tikz}
\newcommand*\circled[1]{\tikz[baseline=(char.base)]{\node[shape=circle,draw,inner sep=1pt] (char) {#1};}}
\newcommand{\loss}{L}

\usepackage{enumitem}
\renewcommand{\texttt}[1]{{\fontfamily{lmtt}\selectfont#1}}

\title{Provably Efficient Reinforcement Learning in Decentralized General-Sum Markov Games\footnotetext{Department of Electrical and Computer Engineering \& Coordinated Science Laboratory, University of Illinois Urbana-Champaign. Email addresses: \{weichao2, basar1\}@illinois.edu. Research leading to this work was supported in part by ONR MURI Grant N00014-16-1-2710 and in part by AFOSR Grant FA9550-19-1-0353. We thank Kaiqing Zhang and Lin F. Yang for helpful discussions. }
}

\makeatletter
\def\@fnsymbol#1{\ensuremath{\ifcase#1\or  \natural \or \dagger\or * \or \ddagger\or
		\mathsection\or \mathparagraph\or \|\or **\or \dagger\dagger
		\or \ddagger\ddagger \else\@ctrerr\fi}}
\makeatother

\author{Weichao Mao  \qquad \qquad \qquad  Tamer Ba\c{s}ar}
\date{\normalsize }

\begin{document}

\maketitle

\begin{abstract}
This paper addresses the problem of learning an equilibrium efficiently in general-sum Markov games through decentralized multi-agent reinforcement learning. Given the fundamental difficulty of calculating a Nash equilibrium (NE), we instead aim at finding a coarse correlated equilibrium (CCE), a solution concept that generalizes NE by allowing possible correlations among the agents' strategies. We propose an algorithm in which each agent independently runs optimistic V-learning (a variant of Q-learning) to efficiently explore the unknown environment, while using a stabilized online mirror descent (OMD) subroutine for policy updates. We show that the agents can find an $\epsilon$-approximate CCE in at most $\widetilde{O}( H^6S A /\epsilon^2)$ episodes, where $S$ is the number of states, $A$ is the size of the largest individual action space, and $H$ is the length of an episode. This appears to be the first sample complexity result for learning in generic general-sum Markov games. Our results rely on a novel investigation of an anytime high-probability regret bound for OMD with a dynamic learning rate and weighted regret, which would be of independent interest. One key feature of our algorithm is that it is fully \emph{decentralized}, in the sense that each agent has access to only its local information, and is completely oblivious to the presence of others. This way, our algorithm can readily scale up to an arbitrary number of agents, without suffering from the exponential dependence on the number of agents. 
\end{abstract}

\section{Introduction}\label{sec:intro} 
Reinforcement learning (RL) has recently shown the capability to solve many challenging sequential decision-making problems, ranging from the game of Go~\citep{silver2016mastering}, Poker~\citep{brown2018superhuman}, and real-time strategy  games~\citep{vinyals2019grandmaster}, to autonomous driving~\citep{shalev2016safe}, and robotics~\citep{kober2013reinforcement}. Many of the RL applications involve the interaction of multiple agents, which are modeled systematically within the framework of multi-agent reinforcement learning (MARL). These success stories have inspired a remarkable line of studies on the theoretical aspects of MARL.

Most of the theoretical efforts in MARL, however, have been devoted to Markov games with special reward structures, such as fully competitive or cooperative games. One prevalent setting is MARL in two-player zero-sum Markov games \citep{wei2017online,xie2020learning}, where the two agents have exactly opposite objectives. Such prevalence is mostly due to the fundamental computational difficulty in more general scenarios: Finding a Nash equilibrium (NE) is known to be PPAD-complete both for two-player general-sum games~\citep{chen2009settling} and zero-sum games with more than two players~\citep{daskalakis2009complexity}. Given the daunting impossibility results, convergence to NE in generic games with no special structure seems hopeless in general. As a result, many important problems in the multi-player general-sum settings, which can model broader and more practical interactive behaviors of decision makers, have been left relatively open. 

In this paper, we make an initial attempt toward understanding some of the theoretical aspects of MARL in decentralized general-sum Markov games. Given the inherent challenges for computing Nash equilibria, we need to target a slightly weaker solution concept than NE. One reasonable alternative is to find a coarse correlated equilibrium (CCE) \citep{moulin1978strategically,aumann1987correlated} of the game. Unlike NE, CCE can always be found in polynomial time for general-sum games \citep{papadimitriou2008computing}, and due to its tractability, calculating CCE has also been commonly used as an important subroutine toward finding Nash equilibria in two-player zero-sum Markov games \citep{xie2020learning,bai2020near}.

Our interest in CCE is mostly motivated by the following folklore result for learning in normal-form games: When the agents independently run no-regret learning algorithms in general-sum normal-form games, their empirical frequency of plays converges to the set of CCE of the game \citep{hart2000simple,cesa2006prediction}. In no-regret learning, each agent independently adapts its policy to minimize the cumulative regret based on only its local information, irrespective of the actions or rewards of the other agents. Well-known examples of no-regret learning algorithms include multiplicative weights update (MWU) \citep{auer1995gambling} and online gradient descent \citep{zinkevich2003online}. Such a folk result hence suggests that CCE is a natural outcome of the simple and \emph{uncoupled} learning dynamics of the agents. A natural question to ask is whether a similar result also holds for Markov games. Specifically, in this paper, we ask the following questions: Can we find CCE in general-sum Markov games using decentralized/uncoupled learning dynamics? If so, can we achieve such a result efficiently, by showing an explicit sample complexity upper bound? 

Before answering these questions, we would like to remark that decentralized MARL in general-sum games can be highly challenging. Specifically, in decentralized learning\footnote{This setting has been studied under various names in the literature, including individual learning \citep{leslie2005individual}, decentralized learning~\citep{arslan2016decentralized}, agnostic learning~\citep{tian2020provably,wei2021last}, and independent learning~\citep{claus1998dynamics,daskalakis2020independent}. It also belongs to a more general  category of teams/games with decentralized information structure~\citep{ho1980team,nayyar2013common,nayyar2013decentralized}.}, each agent only has access to its own local information (history of states, local rewards and local actions), and can neither communicate with other agents nor be coordinated by any central controller during learning. In fact, we require that each agent is completely unaware of the underlying structure of the game, or the presence of the other agents. In decentralized learning, since both the reward and the transition are affected by the other agents, the environment becomes \emph{non-stationary} from each agent's own perspective, especially when the agents learn and update their policies simultaneously. Hence, an agent needs to efficiently explore the unknown environment while bearing in mind that the information it gathered a while ago might no longer be accurate. This makes many successful single-agent RL solutions, which assume that the agent is learning in a stationary Markovian environment, inapplicable. Furthermore, compared with RL in two-player zero-sum games, an additional challenge in general-sum games is \emph{equilibrium selection}. In zero-sum games, all NE have the same value \citep{filar2012competitive}, and there is no ambiguity in defining the sub-optimality of a policy. However, in general-sum games, multiple equilibria can have different values. We hence need to first identify which equilibrium to compare with when we are trying to measure the performance of a policy. Finally, we remark that we consider in this paper the fully observable setup, where the agents have full access to the state information. This is in contrast to the more general team theory with partially observable information structures~\citep{ho1980team,yuksel2013stochastic}, such as those modeled by decentralized partially observable Markov decision processes (DecPOMDPs) \citep{bernstein2002complexity,nayyar2013decentralized}, where each agent has only a private partial observation of the state. Learning or even computing a NE in the latter case is much more challenging, and we do not target such a more general setup in this paper. 

Despite the challenges identified above, we answer both of the aforementioned questions affirmatively, by presenting an algorithm in which the agents can find a CCE in general-sum Markov games efficiently through decentralized learning. Our contributions in this paper are summarized as follows.

\textbf{Contributions.} 1) We study provably efficient exploration in decentralized general-sum Markov games. We propose an algorithm named Optimistic V-learning with Stabilized Online Mirror Descent (V-learning OMD), where V-learning~\citep{bai2020near} is a simple variant of Q-learning. In V-learning OMD, each agent independently runs an optimistic V-learning algorithm to explore the unknown environment, while using an online mirror descent procedure for policy updates. Following the learning process, the CCE can be extracted by simply letting the agents randomly repeat their previous strategies using a common random seed. 2) We show that if all agents in the game run the V-learning OMD algorithm, they can find an $\epsilon$-approximate coarse correlated equilibrium in at most $\widetilde{O}( S A H^6/\epsilon^2)$ episodes, where $S$ is the number of states, $A$ is the size of the largest action space among the agents, and $H$ is the length of an episode. Our result complements its counterpart in normal-form games that uncoupled no-regret learning dynamics lead to CCE.  We further show that our sample complexity is nearly-optimal in that it matches all the parameter dependences in the information-theoretical lower bound, except the horizon length $H$.  3) As an important building block of our analysis, we conduct a novel investigation of a high-probability regret bound for OMD with a dynamic learning rate and weighted regret, which might be of independent interest.  We emphasize that due to the decentralization property, our algorithm readily generalizes to an arbitrary number of agents without suffering from the exponential dependence on the number of agents. Our work appears to be the first to provide non-asymptotic guarantees for MARL in generic general-sum Markov games with efficient exploration, with an additional appealing feature of being decentralized.

\textbf{Outline. } The rest of the paper is organized as follows: We start with a literature review in Section~\ref{sec:related}. In Section~\ref{sec:preliminaries}, we introduce the mathematical model of our problem and necessary preliminaries. In Section~\ref{sec:algorithm}, we present our V-learning OMD algorithm for decentralized learning in general-sum Markov games. A sample complexity analysis of V-learning OMD is given in Section~\ref{sec:analysis}. In Section~\ref{sec:omd}, we analyze a specific adversarial multi-armed bandit problem, which plays a central role in our analysis of the V-learning OMD algorithm. Finally, we conclude the paper in Section~\ref{sec:conclusions}.

\section{Related Work}\label{sec:related}
Multi-agent reinforcement learning is generally modeled within the mathematical framework of stochastic games \citep{shapley1953stochastic} (also referred to as Markov games). Given the PPAD completeness of finding a Nash equilibrium in generic games~\citep{daskalakis2009complexity,chen2009settling}, convergence to NE has mostly been studied in games with special structures, such as two-player zero-sum games or cooperative games. 

In two-player zero-sum Markov games, \citet{littman1994markov} has proposed a Q-learning based algorithm named minimax-Q, whose asymptotic convergence guarantee has later been established in \citet{littman1996generalized}. Recently, various sample efficient methods have been proposed in this setting \citep{wei2017online,bai2020provable,sidford2020solving,xie2020learning,bai2020near,liu2021sharp,zhao2021provably}. Most notably, several works \citep{daskalakis2020independent,tian2020provably,wei2021last} have investigated two-player zero-sum games in a \emph{decentralized} environment similar to ours, where an agent cannot observe the actions of the other agent.  

Another line of research with convergence guarantees is RL in teams or cooperative games. \citet{wang2002reinforcement} have proposed an optimal adaptive learning algorithm that converges to the optimal Nash equilibrium in Markov teams, where the agents share the same objective. More recently, \citet{arslan2016decentralized} have shown that decentralized Q-learning can converge to NE in weakly acyclic games, which cover Markov teams and potential games as important special cases. Later, \citet{yongacoglu2019learning} have further improved \citet{arslan2016decentralized} and achieved convergence to the team-optimal equilibrium. Overall, the aforementioned works have mainly focused on two-player zero-sum games or cooperative games. These results do not carry over in any way to the general-sum games that we consider in this paper.  

A few works have considered games beyond the zero-sum or cooperative settings: \citet{hu2003nash}, \citet{littman2001friend}, and \citet{zehfroosh2020pac} have established convergence guarantees under the assumptions that either a saddle point equilibrium or a coordination equilibrium exists. \citet{greenwald2003correlated} has bypassed the computation of NE in general-sum games by targeting correlated equilibria instead, but no theoretical convergence result has been given. Other approaches for finding NE in general-sum games include minimizing the Bellman-like residuals learned from offline/batch data~\citep{perolat2017learning}, or using a two-timescale algorithm to learn the policy of each player from an optimization perspective~\citep{prasad2015two}. Nevertheless, none of these works has considered sample-efficient exploration in a decentralized environment, a more challenging objective that we pursue in this paper. More recently, \citet{liu2021sharp} have studied the non-asymptotic properties of learning CCE in general-sum Markov games, but their sample complexity bound scales exponentially in the number of agents as a consequence of using a centralized learning approach.

\section{Preliminaries}\label{sec:preliminaries}
An $N$-player (episodic) general-sum Markov game is defined by a tuple \allowbreak $(\mc{N}, H, \mc{S},\{\mc{A}^i\}_{i=1}^N, \{r^i\}_{i=1}^N, P)$, where (1) $\mc{N} = \{1,2,\dots,N\}$ is the set of agents; (2) $H\in\mb{N}_+$ is the number of time steps in each episode; (3) $\mc{S}$ is the finite state space; (4) $\mc{A}^i$ is the finite action space for agent $i\in\mc{N}$; (5) $r^i:[H]\times \mc{S}\times \mc{A} \ra [0,1]$ is the reward function for agent $i$, where $\mc{A} = \times_{i=1}^N \mc{A}^i$ is the joint action space; and (6) $P: [H]\times\mc{S}\times \mc{A} \ra \Delta(\mc{S})$ is the transition kernel. We remark that both the reward function and the state transition function depend on the joint actions of all the agents. 
We assume for simplicity that the reward function is deterministic. Our results can be easily generalized to stochastic reward functions. 

The agents interact with the unknown environment for $K$ episodes, and we let $T = KH$ be the total number of time steps. At each time step, the agents observe the state $s_h \in \mc{S}$, and take actions $a^i_h \in\mc{A}^i, i\in\mc{N}$ simultaneously. We let $a_h =  (a_h^1,\dots, a_h^N) $.  The reward $r^i_h(s_h,a_h)$ is then revealed to agent $i$ privately, and the environment transitions to the next state $s_{h+1}\sim P_h(\cdot | s_h,a_h)$. We assume for simplicity that the initial state $s_1$ of each episode is fixed.

As mentioned before, we focus on the  \emph{decentralized} setting: Each agent only observes the states and its own actions and rewards, but not the actions or rewards of the other agents. In fact, each agent is completely oblivious of the existence of other agents, and is also unaware of the number of agents in the game. The agents are also not allowed to communicate with each other. This decentralized information structure requires an agent to learn to make decisions based on only its local information.

For each agent $i$, a (Markov) policy is a mapping from the time index and the state space to a distribution over its action space. We define the space of policies for agent $i$ by $\Pi^i = \{ \pi^i: [H]\times \mc{S}\ra \Delta(\mc{A}^i) \}$. The joint policy space is denoted by $\Pi = \times_{i=1}^N \Pi^i$. Each agent seeks to find a policy that maximizes its own cumulative reward. A joint policy induces a probability measure on the sequence of states and joint actions. For a joint policy $\pi \in \Pi$, and for each time step $h\in[H]$, state $s \in \mc{S}$, and joint action $a\in \mc{A}$, we define the (state) value function and the state-action value function (Q-function) for agent $i$ as follows: 
\begin{equation}
\begin{aligned}
V_h^{i,\pi} (s) \defeq \ee_\pi \left[ \sum_{h'=h}^{H}r^i_{h'}(s_{h'},a_{h'})\mid s_h = s \right],
Q_h^{i,\pi} (s,a) \defeq \ee_\pi \left[ \sum_{h'=h}^{H}r^i_{h'}(s_{h'},a_{h'})\mid s_h = s, a_h = a \right]. \label{eqn:value} 
\end{aligned}
\end{equation}
For convenience of notations, we use the superscript $-i$ to denote the set of agents excluding agent $i$, i.e., $\mc{N}\backslash \{i\}$. For example, we can rewrite $\pi = (\pi^i,\pi^{-i})$ using this convention.  

\begin{definition} (Best response).\label{def:best_response}
	A policy $\pi^{i \star} \in \Pi^i$ is a \emph{best response} to $\pi^{-i}\in\Pi^{-i}$ for agent $i$ if
	$
		V_1^{i, (\pi^{i \star}, \pi^{-i})}(s) = \sup_{\pi^i} V_1^{i,(\pi^{i }, \pi^{-i})}(s),
	$ 
	for any state $s \in \mc{S}$. 
\end{definition}
The supremum is taken over all \emph{general policies}\footnote{A general policy~\citep{bai2020near} of agent $i$ is a set of maps $\pi^i = \{\pi^i_h : \rr \times (\mc{S}\times \mc{A}^i\times\rr)^{h-1}\times  \mc{S}\ra \Delta(\mc{A}^i)\}_{h\in[H]}$ from a random variable $z\in\rr$ and a history of length $h-1$ to a distribution of actions in $\mc{A}^i$. } that are not necessarily Markov. When $\pi^{-i}\in\Pi^{-i}$ is Markov, the environment is stationary from the perspective of agent $i$, and there always exists a Markov best response. 
\begin{definition} (Nash equilibrium). 
	A Markov joint policy $\pi = (\pi^{i}, \pi^{-i }) \in \Pi$ is  a Nash equilibrium (NE) if $\pi^{i }$ is a best response to $\pi^{-i }$ for all $i\in\mc{N}$.  
\end{definition}

Given the fundamental difficulty of calculating Nash equilibria in general-sum games~\citep{daskalakis2009complexity,chen2009settling}, in this paper, we mainly target a weaker solution concept named coarse correlated equilibrium \citep{moulin1978strategically,aumann1987correlated}. CCE allows possible correlations among the agents' policies. Before we formally define CCE, we need to refine the above definitions of policies to reflect such correlations. We define $\sigma =\{\sigma_h: \rr \times (\mc{S}\times \mc{A}\times \rr)^{h-1}\times\mc{S}\ra \Delta(\mc{A})\}_{h\in[H]}$ as a set of (non-Markov) \emph{correlated policies}, where for each $h\in[H]$, $\sigma_h$ maps from a random variable $z\in\rr$ and a history of length $h-1$ to a distribution over the \emph{joint action} space $\Delta(\mc{A})$. We also assume that the agents can access a common source of randomness (e.g., a common random seed) to sample the random variable $z$. One can see that our definition of correlated policies generalizes the usual concept of policies as we defined earlier, since the latter only allows mappings to each agent's individual action space simplex. Following the new notations, we can analogously define $\sigma^i$ or $\sigma^{-i}$ for each $i\in\mc{N}$ as the proper marginal distributions of $\sigma$. The best responses $\sigma^{i\star}$, value functions $V_h^{i,\sigma}(s)$, and state-action value functions $Q_h^{i,\sigma}(s,a)$ for agent $i$ can be defined similarly as in Definition~\ref{def:best_response} and Equation \eqref{eqn:value}. We are now in a position to define the CCE in an episodic Markov game, as follows:

\begin{definition}\label{def:CCE}
	(Coarse correlated equilibrium). A correlated policy $\sigma$ is a coarse correlated equilibrium if for every agent $i\in\mc{N}$,
	\[
	V_1^{i, \sigma}(s_1) \geq V_1^{i, (\sigma^{i\star},\sigma^{-i})}(s_1).
	\]
\end{definition}

For illustrative purposes, let us compare the definitions of NE and CCE in a simple normal-form game named Hawk-Dove (with no state transitions). There are two players in this game. The row player has the action space $\mc{A} = \{a_1,a_2\}$, and the column player's action space is $\mc{B}=\{b_1,b_2\}$. The reward matrix of the Hawk-Dove game is described in Table~\ref{tbl:game}. There are three Nash equilibria in this game: $(a_1,b_2)$ and $(a_2,b_1)$ are two pure strategy NE, and $((0.5,0.5), (0.5,0.5))$ is a NE in mixed strategies.  Table~\ref{tbl:cce} gives a CCE distribution of the Hawk-Dove game, which assigns equal probabilities to three action pairs: $(a_1,b_1)$, $(a_1,b_2)$, and $(a_2,b_1)$. We can see that NE defines for each player an \emph{independent} probability distribution over a player's own action space; in contrast, a CCE is a probability distribution over the joint action space of the players. In this sense, CCE generalizes NE by allowing possible correlations among the strategies of the agents. In our proposed algorithm, such correlation is implicitly achieved by letting the players use a common random seed. 

\begin{table}[!htb]
	\normalsize
	\def\arraystretch{1.3}
	\begin{minipage}{.5\linewidth}
		\caption{The Hawk-Dove game.}
		\centering
		\begin{tabular}{ccc}
			&           $b_1$            &          $b_2$             \\ \cline{2-3} 
			\multicolumn{1}{c|}{$a_1$}  & \multicolumn{1}{c|}{4,4} & \multicolumn{1}{c|}{1,5} \\ \cline{2-3} 
			\multicolumn{1}{c|}{$a_2$} & \multicolumn{1}{c|}{5,1} & \multicolumn{1}{c|}{0,0} \\ \cline{2-3} 
		\end{tabular}\label{tbl:game}
	\end{minipage}%
	\begin{minipage}{.5\linewidth}
		\centering
		\caption{A CCE in the Hawk-Dove game.}
		\begin{tabular}{ccc}
			&           $b_1$            &          $b_2$             \\ \cline{2-3} 
			\multicolumn{1}{c|}{$a_1$}  & \multicolumn{1}{c|}{1/3} & \multicolumn{1}{c|}{1/3} \\ \cline{2-3} 
			\multicolumn{1}{c|}{$a_2$} & \multicolumn{1}{c|}{1/3} & \multicolumn{1}{c|}{0} \\ \cline{2-3} 
		\end{tabular}\label{tbl:cce}
	\end{minipage} 
\end{table}

Unlike NE, it can be shown that a CCE can always be calculated in polynomial time for general-sum games \citep{papadimitriou2008computing}. If the inequality in Definition~\ref{def:CCE} only holds approximately, then we have an approximate notion of coarse correlated equilibrium.

\begin{definition} \label{def:approxCCE}
	($\epsilon$-approximate CCE). For any $\epsilon>0$, a correlated policy $\sigma$ is an $\epsilon$-approximate coarse correlated equilibrium if for every agent $i\in\mc{N}$, 
	\[
	V_1^{i, \sigma}(s_1) \geq V_1^{i, (\sigma^{i\star},\sigma^{-i})}(s_1) - \epsilon. 
	\]
\end{definition}

For notational convenience, in most parts of the paper we illustrate our algorithms and results for the special case of two-player general-sum games, i.e., $N=2$. It is straightforward to extend our results to the general $N$-player games as we defined above. With two players, we use $\mc{A}$ and $\mc{B}$ to denote the action spaces of players 1 and 2, respectively. Let $S = |\mc{S}|$, $A = |\mc{A}|$ and $B = |\mc{B}|$. We assume that an upper bound $\max\{A,B\}$ on the sizes of the action spaces is common knowledge. We also rewrite the correlated policies $(\sigma^1,\sigma^2)$ as $(\mu,\nu)$.

\section{The V-Learning OMD Algorithm}\label{sec:algorithm} 
In this section, we introduce our algorithm Optimistic V-learning with Stabilized Online Mirror Descent (V-learning OMD) for decentralized multi-agent RL in general-sum Markov games.  

V-learning OMD naturally integrates the idea of optimistic V-learning in single-agent RL~\citep{jin2018q} with Online Mirror Descent (OMD)~\citep{nemirovskij1983problem,zinkevich2003online} in online convex optimization. First, our algorithm uses optimistic V-learning to efficiently explore the unknown environment, as in single-agent RL. Second, each agent selects its actions following a no-regret OMD algorithm in order to achieve a CCE. The intuition of using no-regret learning here is to defend against the unobserved behavior of the opponents, by presuming that the opponents' behavior will impair the reward sequence arbitrarily. Seemingly conservative, we will show  that this suffices to find the CCE. The use of no-regret learning is also reminiscent of the well-known result in normal-form games that if all agents run a no-regret learning algorithm, the empirical frequency of their actions converge to a CCE \citep{cesa2006prediction}. These components also make our algorithm decentralized, which can be implemented individually using only the local rewards received and the local actions executed, without any communication among the agents.

\begin{algorithm*}[!tbp]
	\textbf{Define:} $F(\theta) = \sum_{a=1}^A (\theta(a) \log(\theta(a)) - \theta(a))$ for $\theta\in\rr_+^A$, $D_F (u,v) = F(u) - F(v) - \inner{u-v, \grad F(v)}$ for $u,v\in\rr_+^A$.
	
	\textbf{Initialize:} $\overline{V}_h(s) =V_h(s)\gets H-h+1, N_h(s)\gets 0, \theta_h(a\mid s)\gets 1/A$, $\forall h\in[H+1],s\in\mc{S},a\in\mc{A}$. 
	
	\For{episode $k\gets 1$ to $K$}
	{
		Receive $s_1$\;
		\For{step $h\gets 1$ to $H$}
		{
			Take action $a_h \sim \theta_h(\cdot \mid s_h)$\;
			Observe reward $r_h$ and next state $s_{h+1}$\;
			$N_h(s_h) \gets N_h(s_h) + 1,t \gets N_h(s_h)$\;
			$\alpha_t \gets \frac{H+1}{H+t}, \beta_t \gets c\sqrt{\frac{H^4 A \iota}{t}}, \gamma_t \gets \sqrt{\frac{\log A}{At}}, \eta_t \gets \sqrt{\frac{\log A}{At}}$\;
			$V_{h}\left(s_{h}\right) \gets  \left(1-\alpha_{t}\right) V_{h}\left(s_{h}\right)+\alpha_{t}\left(r_{h}+\ol{V}_{h+1}\left(s_{h+1}\right)+\beta_{t}\right)$\;
			$\ol{V}_h(s_h)\gets \min\{V_h(s_h), H-h+1\}$\;
			\For{action $a \in \mc{A}$}
			{
				$\hat{l}_{h}(s_h, a) \leftarrow (H-r_h-\ol{V}_{h+1}(s_{h+1})) \indicator\left\{a_{h}=a\right\} /\left(\theta_{h}(a \mid s_h)+\gamma_{t}\right)$\;
			}
			$\theta' \gets \arg\min_{\theta \in \Delta(\mc{A})} \left\{ \eta_t \inner{\theta,  \hat{l}_h(s_h,\cdot )} + D_F(\theta, \theta_h(\cdot \mid s_h))\right\}$\;
			$\theta_h(\cdot \mid s_h)\gets \lambda_t \theta' + (1-\lambda_t) \mbf{1}/A$, where $\lambda_t = \frac{\eta_{t+1} \alpha_t (1-\alpha_{t+1})}{\eta_t \alpha_{t+1}}$\;
		}
	}
	\caption{Optimistic V-learning with Stabilized Online Mirror Descent (V-learning OMD)}\label{alg:qomd}
\end{algorithm*}

The algorithm run by agent 1 (with action space $\mc{A}$) is presented in Algorithm~\ref{alg:qomd}. The algorithm for agent 2 (or other agents in the setting with more than two agents) is symmetric, by simply replacing the action space $\mc{A}$ with the agent's own action space. We thus omit the index of an agent in the notations for clarity. We use $\theta_h(a\mid s_h)$ to denote the probability of taking action $a$ at state $s_h$ and step $h$, where $\theta_h(\cdot \mid s_h)\in\Delta(\mc{A})$. At each step $h$ of an episode, the agent first takes an action $a_h$ according to a policy $\theta_h(\cdot \mid s_h)$ for the current state $s_h$, and observes the reward $r_h$ and the next state $s_{h+1}$. It also counts the number of times $t \defeq N_h(s_h)$ that state $s_h$ has been visited, and constructs a bonus term $\beta_t = c\sqrt{\frac{H^4 A \iota}{t}}$ ($c$ is some absolute constant and $\iota$ is a log factor to be defined later) that is used to upper bound the state value function. The agent then updates the optimistic state value functions by:
\begin{equation}\label{eqn:updaterule}
V_{h}\left(s_{h}\right) \gets \left(1-\alpha_{t}\right) V_{h}\left(s_{h}\right)+\alpha_{t}\left(r_{h}+\ol{V}_{h+1}\left(s_{h+1}\right)+\beta_{t}\right),
\end{equation}
where the learning rate is $\alpha_t = (H+1)/(H+t)$. This update rule essentially follows the optimistic Q-learning algorithm~\citep{jin2018q} in the single-agent scenario, except that instead of estimating the Q-functions, we maintain optimistic estimates of the state value functions. This is because the definition of $Q(s,a)$ explicitly depends on the joint actions of all the agents, which cannot be observed in a decentralized environment. Such an argument is also consistent with the Optimistic Nash V-learning~\citep{bai2020near} and the V-OL~\citep{tian2020provably} algorithms for RL in two-player zero-sum games. 

Unlike RL in the single-agent problem where the agent takes an action with the largest optimistic Q-function, in a multi-agent  environment, the agent proceeds more conservatively by running an adversarial bandit algorithm to account for the unobserved effects of other agents' policy changes. At each step $h\in[H]$ and each state $s_h\in\mc{S}$, we use a variant of online mirror descent with bandit feedback to compute a policy $\theta_h(\cdot \mid s_h)$. OMD is an iterative process that computes the current policy by carrying out a simple gradient update in the dual space, where the dual space is defined by a mirror map (or a regularizer) $F$. In our algorithm, we use a standard unnormalized negentropy regularizer $F(\theta) = \sum_{a=1}^A (\theta(a) \log(\theta(a)) - \theta(a))$ for $\theta\in \rr_+^A$. Given a mirror map $F$, the $F$-induced Bregman divergence is defined as
$
D_F (u,v) = F(u) - F(v) - \inner{u-v, \grad F(v)}. 
$ 
Given the (bandit-feedback) loss vector $\hat{l}_h(s_h,\cdot)$ at step $h$ and state $s_h$, the OMD update rule is given by (Line 13 of Algorithm~\ref{alg:qomd}):
\[
\theta^{\text{new}}(\cdot \mid s_h) \gets \arg\min_{\theta \in \Delta(\mc{A})} \left\{ \eta_t \inner{\theta,  \hat{l}_h(s_h,\cdot )} + D_F(\theta, \theta^{\text{old}}(\cdot \mid s_h))\right\},
\]
where $\eta_t = \sqrt{\log A/(At)}$ is the learning rate. We remark that OMD itself is a well-developed algorithmic framework with a rich literature. But in our case, to be consistent with the changing learning rate in the V-learning part and the high-probability nature of the sample complexity bounds, we additionally require an OMD algorithm to have (1) a dynamic learning rate and (2) a high probability regret bound, with respect to (3) a weighted definition of regret.  Such a result is absent in the literature as far as we know. Interestingly, incorporating OMD with a dynamic learning rate is an active and challenging sub-area per se: An impossibility result~\citep{orabona2018scale} has shown that standard OMD with an $\eta_t \propto \sqrt{1/t}$ learning rate can incur linear regret when the Bregman divergence is unbounded, which actually covers our choice of $D_F$. A stabilization technique~\citep{fang2020online} was later introduced to resolve this problem, by replacing the policy at each step with a convex combination of this policy and the initial policy. This stabilization technique is also helpful in our method (Line 14 in Algorithm~\ref{alg:qomd}), although the design of the convex combination is a little more involved due to the weighted regret. We provide a more detailed description of the bandit subroutine and an analysis of our OMD algorithm in Section~\ref{sec:omd}.

\section{Theoretical Analyses}\label{sec:analysis}
In this section, we present our main results on the sample complexity upper bound of V-learning OMD, and characterize the fundamental limits of the problem by providing a lower bound. 

We first introduce a few notations to facilitate the analysis. For a given step $h\in[H]$ of episode $k\in[K]$, we denote by $s_h^k$ the state that the agents observe at this step. Let $\mu_h^k: \mc{S}\ra \Delta(\mc{A})$ and $\nu_h^k: \mc{S}\ra \Delta(\mc{B})$ be the (interim) strategies at step $h$ of episode $k$ specified by $\theta_h$ in Algorithm~\ref{alg:qomd} to agents 1 and 2, respectively. Let $a_h^k\in\mc{A}$ and $b_h^k\in\mc{B}$ be the actual actions taken by the two agents.  Let $\ol{V}_h^k(s_h^k),V_h^k(s_h^k)$, and $N_h^k(s_h^k)$, respectively, be the values of $\ol{V}_h(s_h), V_h(s_h)$, and $N_h(s_h)$ in Algorithm~\ref{alg:qomd} calculated by agent 1 at the \emph{beginning} of the $k$-th episode. Symmetrically, define $\widetilde{V}_h^k(s_h^k)$ to be the value of $\ol{V}_h(s_h)$ calculated by agent 2, which does not necessarily take the same value as $\ol{V}_h^k(s_h^k)$.  For notational convenience, we often suppress  the sub/super-scripts $(h,k)$ when there is no possibility of any ambiguity. When the state $s_h^k$ is clear from the context, we also sometimes abbreviate $N_h^k(s_h^k)$ as $n_h^k$ or even simply as $t$. For a fixed state $s\in\mc{S}$, let $t = N_h^k(s)$, and suppose that $s$ was visited at episodes $k^1 < k^2 < \cdots < k^t$ at the $h$-th step before the $k$-th episodes. If we further define $\alpha_{t}^{0}\defeq\prod_{j=1}^{t}\left(1-\alpha_{j}\right)$ and $\alpha_{t}^{i}\defeq \alpha_{i} \prod_{j=i+1}^{t}\left(1-\alpha_{j}\right)$, one can show that the update rule in~\eqref{eqn:updaterule} can be equivalently expressed as
\begin{equation}\label{eqn:updaterulenew}
V_{h}^{k}(s)=\alpha_{t}^{0} (H-h+1)+\sum_{i=1}^{t} \alpha_{t}^{i}\left[r_{h}\left(s, a_{h}^{k^{i}}, b_{h}^{k^{i}}\right)+\ol{V}_{h+1}^{k^{i}}\left(s_{h+1}^{k^{i}}\right)+\beta_{i}\right]. 
\end{equation}
This update rule follows the standard optimistic Q-learning algorithm~\citep{jin2018q} in single-agent RL, and has also appeared in RL for two-player zero-sum games~\citep{bai2020near}. In the following lemma, we recall several properties of $\alpha_t^i$ that are useful in our analysis.

\begin{lemma}\label{lemma:alpha}
	(Properties for $\alpha_t^i$, Lemma 4.1 in~\cite{jin2018q}). 
	\begin{enumerate}
		\item $\sum_{i=1}^{t} \alpha_{t}^{i}=1 \text { and } \alpha_{t}^{0}=0 \text { for } t \geq 1.$
		\item $\sum_{i=1}^{t} \alpha_{t}^{i}=0 \text { and } \alpha_{t}^{0}=1 \text { for } t=0$. 
		\item $\frac{1}{\sqrt{t}} \leq \sum_{i=1}^{t} \frac{\alpha_{t}^{i}}{\sqrt{i}} \leq \frac{2}{\sqrt{t}} \text { for every } t \geq 1 $.
		\item $\max _{i \in[t]} \alpha_{t}^{i} \leq \frac{2 H}{t} \text { and } \sum_{i=1}^{t}\left(\alpha_{t}^{i}\right)^{2} \leq \frac{2 H}{t} \text { for every } t \geq 1$.
		\item $\sum_{t=i}^{\infty} \alpha_{t}^{i}=1+\frac{1}{H} \text { for every } i \geq 1 $. 
	\end{enumerate}
\end{lemma}

\begin{algorithm*}[!tbp]
	\textbf{Require: }A common random seed shared by both agents. 
	
	\textbf{Input:} The strategy trajectories $\{(\mu_h^k, \nu_h^k)\}_{h=1,k=1}^{H,K}$ specified by Algorithm~\ref{alg:qomd}. 
	
	\For{step $h'\gets h$ to $H$}
	{
		Receive $s_{h'}$\;
		$t\gets N_{h'}^k(s_{h'})$\;
		Sample $m$ from $[t]$ with $\pp(m=i) = \alpha_t^i$ using the common random seed\;
		Let $k$ be the index of the episode in which $s_{h'}$ was visited for the $m$-th time during the execution of Algorithm~\ref{alg:qomd}\;
		Execute the strategy pair $(\mu_{h'}^k(\cdot \mid s_{h'}), \nu_{h'}^k(\cdot \mid s_{h'}) )$\; 
	}
	\caption{Construction of $(\bar{\mu}_h^k,\bar{\nu}_h^k)$ }\label{alg:correlated}
\end{algorithm*}

Based on the strategy trajectories $\{(\mu_h^k, \nu_h^k)\}_{h=1,k=1}^{H,K}$ of the two agents specified by Algorithm~\ref{alg:qomd}, we construct an auxiliary pair of correlated policies $(\bar{\mu}_h^k, \bar{\nu}_h^k)$ for each $(h,k)\in[H]\times [K]$. The construction of such correlated policies, largely inspired by the construction of the ``certified policies'' in~\citet{bai2020near}, is formally defined in Algorithm~\ref{alg:correlated}. Such auxiliary correlated policies will play a significant role throughout our analysis, and are closely related to the CCE correlated policy that we will construct later. In words, $(\bar{\mu}_h^k, \bar{\nu}_h^k)$ proceeds as follows: It first observes the current state $s_h$, and let $t = N_h^k(s_h)$. Then, it randomly samples an episode index $k^j$ from $\{ k^1,k^2,\dots,k^t \}$, the set of episodes in which the state $s_h$ was previously visited during the execution of the first $k$ episodes of Algorithm~\ref{alg:qomd}. Each index $k^i$ has a probability of $\alpha_t^i$ to be selected. It is easy to verify that $\sum_{i=1}^t \alpha_t^i = 1$, and hence we have specified a well-defined probability distribution over the episode index set. Finally, $(\bar{\mu}_h^k, \bar{\nu}_h^k)$ executes the sampled strategy $(\mu_{h}^k(\cdot \mid s_{h}), \nu_{h}^k(\cdot \mid s_{h}) )$ at step $h$, and then repeats a similar procedure using $(\bar{\mu}_{h+1}^{k^j}, \bar{\nu}_{h+1}^{k^j})$ at step $h+1$, and so on. 

From the collection of such auxiliary correlated policies $\{(\bar{\mu}_h^k, \bar{\nu}_h^k)\}_{h=1,k=1}^{H,K}$, we finally construct a correlated policy $(\bar{\mu},\bar{\nu})$, which we will show later is a CCE. A detailed description of the construction of   $(\bar{\mu},\bar{\nu})$ is presented in Algorithm~\ref{alg:certify}. By construction, $(\bar{\mu},\bar{\nu})$ first uniformly samples an index $k$ from $[K]$ using a common random seed, and then proceeds by following the auxiliary correlated policy $(\bar{\mu}_1^k,\bar{\nu}_1^k)$. One can see that the notations we have defined are related through the following equation: $V_1^{\bar{\mu},\bar{\nu}}(s_1) = \frac{1}{K}\sum_{k=1}^K V_1^{\bar{\mu}_1^k,\bar{\nu}_1^k}(s_1)$. We also remark that the common random seed used in Algorithms~\ref{alg:correlated} and~\ref{alg:certify} implicitly plays the role of the ``trusted coordinator'' typically used in the language of correlated equilibria.

For notational convenience, we further introduce the operator $\mathbb{P}_{h} V(s, a, b)= \mathbb{E}_{s^{\prime} \sim P_{h}(\cdot \mid s, a, b)} V\left(s^{\prime}\right)$ for any value function $V$, and $\mathbb{D}_{\mu_h\times\nu_h} Q(s)= \mathbb{E}_{(a, b) \sim (\mu_h\times\nu_h)} \allowbreak Q(s, a, b)$ for any strategy pair $(\mu_h,\nu_h)$ and any state-action value function $Q$. With these notations, for any $(s,a,b,h)\in \mc{S}\times \mc{A}\times \mc{B}\times [H]$ and for any policy pair $(\mu,\nu)$, the Bellman equations can be rewritten more succinctly as
$
Q_{h}^{\mu, \nu}(s, a, b)=\left(r_{h}+\mathbb{P}_{h} V_{h+1}^{\mu, \nu}\right)(s, a, b),$ and $V_{h}^{\mu, \nu}(s)=\left(\mathbb{D}_{\mu_{h} \times \nu_{h}} Q_{h}^{\mu, \nu}\right)(s). 
$
Recalling the definitions of the best responses, we further define $V_{k,H+1}^{\star,\bar{\nu}_{H+1}^k}(s) = 0,\forall k\in[K], s\in\mc{S}$. Then, we know that for each $(k,h,s)\in [K]\times[H]\times \mc{S}$, 
\begin{equation}\label{eqn:weighted}
V_{k,h}^{\star, \bar{\nu}_h^k}(s)\leq  \max_{\mu_h} \sum_{i=1}^t \alpha_t^i \mb{D}_{\mu_h\times \nu_h^{k^i}} \l r_h + \mb{P}_hV_{k^i,h+1}^{\star, \bar{\nu}_{h+1}^k}\r  (s),
\end{equation}
In what follows, we will simply write $V_{k,h}^{\star, \bar{\nu}_h^k}$ as $V_{k,h}^{\star, \bar{\nu}}$ for notational convenience, because the time step $(h,k)$ is always clear from the subscripts. We can also define $V_{k,h}^{\bar{\mu},\star }(s)$ analogously.

\begin{algorithm*}[!tbp]
	\textbf{Require: }A common random seed shared by both agents. 
	
	\textbf{Input:} The strategy trajectories $\{(\mu_h^k, \nu_h^k)\}_{h=1,k=1}^{H,K}$ specified by Algorithm~\ref{alg:qomd}. 
	
	Uniformly sample $k$ from $[K]$ using the common random seed. 
	
	\For{step $h\gets 1$ to $H$}
	{
		Receive $s_h$\;
		$t\gets N_h^k(s_h)$\;
		Sample $m$ from $[t]$ with $\pp(m=i) = \alpha_t^i$ using the common random seed\;
		Let $k$ be the index of the episode in which $s_h$ was visited for the $m$-th time during the execution of Algorithm~\ref{alg:qomd}\;
		Execute the strategy pair $(\mu_h^k(\cdot \mid s_h), \nu_h^k(\cdot \mid s_h) )$\; 
	}
	\caption{Construction of the Correlated Policy $(\bar{\mu},\bar{\nu})$}\label{alg:certify}
\end{algorithm*}

\begin{remark}
	Our definition of the correlated policy is inspired by the ``certified policies'' \citep{bai2020near} for learning in two-player zero-sum Markov games, but with additional challenges to address: In the zero-sum setting, the Nash equilibrium value is always unique, and the regret with respect to the equilibrium value can be easily defined a priori (by means of the ``duality gap''). But in general-sum games, the equilibrium value is not necessarily unique. We hence need to first specify an equilibrium before we are able to define the regret. In our analysis, the equilibrium value we choose is the one associated with the correlated policy $(\bar{\mu},\bar{\nu})$. In addition, we also emphasize that the correlated policy is only used for analytical purposes; the actual strategies adopted by the agents during the execution of Algorithm~\ref{alg:qomd} are still $\{(\mu_h^k, \nu_h^k)\}$. 
\end{remark}

We start with an intermediate result, which states that the optimistic $\ol{V}_h^k(s)$ and $\widetilde{V}_h^k(s)$ values are indeed high-probability upper bounds of $V_{k,h}^{\star, \bar{\nu}}(s)$ and $V_{k,h}^{\bar{\mu},\star }(s)$, respectively. The proof relies on a delicate investigation of a high-probability regret bound for OMD with a dynamic learning rate, which we will elaborate on in the next section. 

\begin{lemma}\label{lemma:upperbound}
	For any $p\in (0,1]$, let $\iota = \log(2S\max\{A,B\}T/p)$. It holds with probability at least $1-p$ that
	$\overline{V}_h^k(s)\geq V_{k,h}^{\star, \bar{\nu}}(s)$ and $\widetilde{V}_h^k(s)\geq V_{k,h}^{\bar{\mu},\star}(s)$,  for all $(s,h,k)\in \mc{S}\times [H]\times[K]$.
\end{lemma}
\begin{proof}
	In the following, we provide a proof for the first inequality. The second inequality can be shown using a similar argument. 
	
	Notice that it suffices to show $V_h^k(s)\geq V_{k,h}^{\star, \bar{\nu}}(s)$, because $\ol{V}_h^k(s) = \min\{V_h^k(s),H-h+1\}$, and $V_{k,h}^{\star, \bar{\nu}}(s) \leq H-h+1$ always holds. Our proof relies on backward induction on $h\in [H]$. First, the claim holds for $h = H+1$ by the definition of ${V}_{H+1}^k(s)$. Now, suppose	${V}_{h+1}^k(s)\geq V_{k,h+1}^{\star, \bar{\nu}}(s)$ for all $s\in\mc{S}$. By the definition of $V_{k,h}^{\star, \bar{\nu}}(s)$ and the induction hypothesis,
	\begin{align}
	V_{k,h}^{\star, \bar{\nu}}(s)\leq &\max_{\mu_h} \sum_{i=1}^t \alpha_t^i \mb{D}_{\mu_h\times \nu_h^{k^i}} \l r_h + \mb{P}_hV_{k^i,h+1}^{\star, \bar{\nu}}\r  (s)\nonumber\\
	\leq & \max_{\mu_h} \sum_{i=1}^t \alpha_t^i \mb{D}_{\mu_h\times \nu_h^{k^i}} \l r_h + \mb{P}_h \ol{V}_{h+1}^{k^i}\r  (s).\label{eqn:a3}
	\end{align}
	Further, define
	\begin{equation}\label{eqn:a2}
	R_t = \max _{\mu_h} \sum_{i=1}^{t} \alpha_{t}^{i} \mathbb{D}_{\mu_h \times \nu_{h}^{k^i}}\big(r_{h}+\mathbb{P}_{h} \ol{V}_{h+1}^{k^{i}}\big)(s)-\sum_{i=1}^{t} \alpha_{t}^{i} \mathbb{D}_{\mu_{h}^{k^i} \times \nu_{h}^{k^{i}}}\big(r_{h}+\mathbb{P}_{h} \ol{V}_{h+1}^{k^{i}}\big)(s).
	\end{equation}
	One may observe that from the perspective of player 1, $R_t$ is the weighted sum of the differences between the actual value that player 1 collected for the first $t$ times that state $s$ is visited, and the value that could have been achieved using the best fixed policy in hindsight. $R_t$ can hence be thought of as the weighted regret of an adversarial bandit problem, which we will formally present and analyze in Section~\ref{sec:omd}. Specifically, the loss function of the bandit problem is defined as 
	\[
	l_{i}(a)=\mathbb{E}_{b \sim \nu_{h}^{k^i}(s)}\left\{H-h+1-r_{h}(s, a, b)-\mathbb{P}_{h} \ol{V}_{h+1}^{k^{i}}(s, a, b)\right\}.
	\]
	The weight of the regret at round $i$ is $w_i = \alpha_t^i$. If we define 
	\[
	\mu_h^\star \defeq \arg\min_{\mu_h} \sum_{i=1}^t w_i \inner{\mu_h, l_i} = \arg\max_{\mu_h} \sum_{i=1}^{t} \alpha_{t}^{i} \mathbb{D}_{\mu_h \times \nu_{h}^{k^i}}\left(r_{h}+\mathbb{P}_{h} \ol{V}_{h+1}^{k^{i}}\right)(s),
	\]
	then, $R_t$ can be equivalently rewritten as 
	\[
	R_{t}=\sum_{i=1}^{t} w_{i}\left\langle\mu_{h}^{\star}-\mu_{h}^{k^{i}}, l_{i}\right\rangle.
	\]
	Later in Section~\ref{sec:omd}, we will analyze an adversarial bandit problem in exactly the same form. Applying the regret bound (which will be presented in Theorem~\ref{thm:omd_regret} of Section~\ref{sec:omd}) of this bandit problem, we obtain the following result with probability at least $1-p/(2SHK)$:
	\begin{align}
	R_{t} \leq& 2 H \alpha_{t}^{t} \sqrt{A t \iota}+\frac{3 H \sqrt{A \iota}}{2} \sum_{i=1}^{t} \frac{\alpha_{t}^{i}}{\sqrt{i}}+\frac{1}{2} H \alpha_{t}^{t}\iota +H \sqrt{2 \iota \sum_{i=1}^{t}\left(\alpha_{t}^{i}\right)^{2}}\nonumber \\
	\leq & 4 H^{2} \sqrt{A \iota / t}+3 H \sqrt{A\iota / t}+H^{2} \iota / t+\sqrt{4 H^{3} \iota / t}\nonumber\\
	\leq & 10H^2 \sqrt{A\iota /t},\label{eqn:a1}
	\end{align}
	where in the first step we have used the fact that $w_i$ is increasing and $\max_{i \leq t} w_i = \alpha_t^t$, and the second step is due to Lemma~\ref{lemma:alpha}.

	Finally, let $\mc{F}_i$ be the $\sigma$-algebra generated by all the random variables before episode $k^i$. Then, we can see that $\{ r_h(s,a_h^{k^i}, b_h^{k^i}) + \ol{V}_{h+1}^{k^i} (s_{h+1}^{k^i}) \}_{i=1}^t$ is a martingale with respect to $\{\mc{F}_i\}_{i=1}^t$. From the Azuma-Hoeffding inequality, it holds with probability at least $1-p/(2SHK)$ that
	\begin{equation}\label{eqn:a4}
	\begin{aligned}
	&\sum_{i=1}^{t} \alpha_{t}^{i} \mathbb{D}_{\mu_{h}^{k^i} \times \nu_{h}^{k^{i}}}\left(r_{h}+\mathbb{P}_{h} \ol{V}_{h+1}^{k^{i}}\right)(s)\\
	\leq& \sum_{i=1}^{t} \alpha_{t}^{i}\left[r_{h}\left(s, a_{h}^{k^{i}}, b_{h}^{k^{i}}\right)+\ol{V}_{h+1}^{k^{i}}\left(s_{h+1}^{k^{i}}\right)\right] + 2 \sqrt{H^{3} \iota / t}, 
	\end{aligned}
	\end{equation}
	where $\iota$ suppresses logarithmic terms. 
	Finally, combining the results in~\eqref{eqn:a3}, \eqref{eqn:a2}, \eqref{eqn:a1}, \eqref{eqn:a4}, and applying a union bound, we obtain that
	\[
	\begin{aligned}
	V_{k,h}^{\star, \bar{\nu} }(s) \leq&  \max_{\mu_h} \sum_{i=1}^t \alpha_t^i \mb{D}_{\mu_h\times \nu_h^{k^i}} \l r_h + \mb{P}_h \ol{V}_{h+1}^{k^i}\r  (s)\\
	\leq & \sum_{i=1}^{t} \alpha_{t}^{i} \mathbb{D}_{\mu_{h}^{k^i} \times \nu_{h}^{k^{i}}}\left(r_{h}+\mathbb{P}_{h} \ol{V}_{h+1}^{k^{i}}\right)(s) + 10 H^2 \sqrt{A\iota /t }\\
	\leq & \sum_{i=1}^{t} \alpha_{t}^{i}\left[r_{h}\left(s, a_{h}^{k^{i}}, b_{h}^{k^{i}}\right)+\ol{V}_{h+1}^{k^{i}}\left(s_{h+1}^{k^{i}}\right)\right] + 10 H^2 \sqrt{A\iota /t } + 2\sqrt{H^3\iota/t}\\
	\leq & \alpha_{t}^{0} H+\sum_{i=1}^{t} \alpha_{t}^{i}\left[r_{h}\left(s, a_{h}^{k^{i}}, b_{h}^{k^{i}}\right)+\ol{V}_{h+1}^{k^{i}}\left(s_{h+1}^{k^{i}}\right)+\beta_{i}\right]\\
	= & {V}_h^k(s),
	\end{aligned}
	\]
	where the second to last step is by the definition of $\beta_t = c\sqrt{\frac{H^4 A \iota}{t}}$ for some large constant $c$, and Lemma~\ref{lemma:alpha}. In the last step we used the formulation of ${V}_h^k(s)$ in \eqref{eqn:updaterulenew}. This completes the proof of the induction step. 
\end{proof}

By construction of the auxiliary correlated policies $(\bar{\mu}_h^k,\bar{\nu}_h^k)$, we know that for any $(s,h,k)\in\mc{S}\times [H]\times [K]$, the corresponding value function can be written recursively as follows:
\[
V_{k,h}^{\bar{\mu}, \bar{\nu}}(s) = \sum_{i=1}^{t} \alpha_{t}^{i} \mathbb{D}_{\mu_{h}^{k^i} \times \nu_{h}^{k^{i}}}\left(r_{h}+\mathbb{P}_{h} V_{k^i,h+1}^{\bar{\mu}, \bar{\nu}}\right)(s), 
\]
and $V_{k,H+1}^{\bar{\mu}, \bar{\nu}}(s) = 0$ for any $k\in[K],s\in\mc{S}$, where again notice that we have dropped the dependence on $(h,k)$. The following result shows that, on average, the agents have no incentive to deviate from the correlated policies, up to a regret term of the order $\widetilde{O}(\sqrt{H^6 SA/K})$. 

\begin{theorem}\label{thm:main}
	For any $p\in(0,1]$, let $\iota = \log(2S\max\{A,B\}T/p)$. With probability at least $1-p$, 
	$$
	\begin{aligned}
	&\frac{1}{K}\sum_{k=1}^K \l V_{k,1}^{\star, \bar{\nu}}(s_1) -  V_{k,1}^{\bar{\mu}, \bar{\nu}}(s_1)\r\leq O(\sqrt{H^6 S A \iota/K}),   \text{ and } \\
	&\frac{1}{K}\sum_{k=1}^K \l V_{k,1}^{\bar{\mu},\star}(s_1) - V_{k,1}^{\bar{\mu}, \bar{\nu}}(s_1)\r \leq  O(\sqrt{H^6 S B \iota/K}).
	\end{aligned}
	$$  
\end{theorem}
\begin{proof}
	We provide a proof for the first bound. The second one can be shown using a similar argument. For analytical purposes, we introduce two new notations $\underline{V}$ and $\ut{V}$ that serve as lower confidence bounds of the value estimates for agent 1. Specifically, for any $(s,h,k)\in\mc{S}\times [H+1]\times [K]$, we define $\ul{V}_h^k(s) = \ut{V}_h^k(s) = 0$ if $h=H+1$ or the $(h,s)$ pair has not been visited before episode $k$, and otherwise define 
	\[
	\ut{V}_{h}^{k}(s)=\sum_{i=1}^{t} \alpha_{t}^{i}\left[r_{h}\left(s, a_{h}^{k^{i}}, b_{h}^{k^{i}}\right)+\ul{V}_{h+1}^{k^{i}}\left(s_{h+1}^{k^{i}}\right)-\beta_{i}\right],\text{ and } \ul{V}_h^k(s) = \max\{\ut{V}_h^k(s), 0\}.
	\]
	Notice that these two notations are only introduced for ease of analysis, and the agent does not need to explicitly maintain such values during the learning process. In the following, we show that $\ul{V}_h^k(s) \leq V_{k,h}^{\bar{\mu},\bar{\nu}}(s)$, for all $(s,h,k)\in\mc{S}\times[H]\times[K]$. Again, it suffices to show that $\ut{V}_h^k(s) \leq V_{k,h}^{\bar{\mu},\bar{\nu}}(s)$, because $\ul{V}_h^k(s) = \max\{\ut{V}_h^k(s), 0\}$, and $V_{k,h}^{\bar{\mu},\bar{\nu}}(s) \geq 0$ always holds. Our proof relies on backward induction on $h\in[H]$. The claim trivially holds for $h=H+1$. Suppose $\ut{V}_{h+1}^k(s) \leq V_{k,h+1}^{\bar{\mu},\bar{\nu}}(s)$ for all $s\in\mc{S}$. By the definition of $\ut{V}_h^k(s)$,
	\[
	\begin{aligned}
	\ut{V}_{h}^{k}(s)=&\sum_{i=1}^{t} \alpha_{t}^{i}\left[r_{h}\left(s, a_{h}^{k^{i}}, b_{h}^{k^{i}}\right)+\ul{V}_{h+1}^{k^{i}}\left(s_{h+1}^{k^{i}}\right)-\beta_{i}\right]\\
	\leq & \sum_{i=1}^{t} \alpha_{t}^{i}  \mathbb{D}_{\mu_{h}^{k^i} \times \nu_{h}^{k^{i}}}\left(r_{h}+\mathbb{P}_{h} \ul{V}_{h+1}^{k^{i}}\right)(s)\\
	\leq & \sum_{i=1}^{t} \alpha_{t}^{i}  \mathbb{D}_{\mu_{h}^{k^i} \times \nu_{h}^{k^{i}}}\left(r_{h}+\mathbb{P}_{h} V_{k^i,h+1}^{\bar{\mu},\bar{\nu}}\right)(s)\\
	=& V_{k,h}^{\bar{\mu},\bar{\nu}}(s).
	\end{aligned}
	\]
	where the second step uses the Azuma-Hoeffding inequality and the definition of $\beta_i$, and the third step is by the induction hypothesis. This completes the proof of the induction. 
	
	Together with Lemma~\ref{lemma:upperbound}, we know that 
	$$
	\sum_{k=1}^K \l V_{k,1}^{\star, \bar{\nu}}(s_1) -  V_{k,1}^{\bar{\mu}, \bar{\nu}}(s_1)\r \leq \sum_{k=1}^K \l \ol{V}_1^{k}(s_1) -  \ul{V}_1^{k}(s_1)\r,
	$$
	and so we only need to find an upper bound for the RHS. Define $\delta_h^k \defeq \overline{V}_h^k(s_h^k)- \ul{V}_h^k(s_h^k)$. The main idea of the proof is similar to optimistic Q-learning in the single-agent setting~\citep{jin2018q}: We seek to upper bound $\sum_{k=1}^K \delta_h^k$ by the next step $\sum_{k=1}^K \delta_{h+1}^k$, and then obtain a recursive formula. 
	
	By the definitions of $\overline{V}_h^k(s_h^k)$ and $\ul{V}_h^k(s_h^k)$, we know that
	\[
	\begin{aligned}
	\delta_h^k = &\overline{V}_h^k(s_h^k)- \ul{V}_h^k(s_h^k)\\
	\leq & \alpha_{t}^{0} H+\sum_{i=1}^{t} \alpha_{t}^{i}\left[\ol{V}_{h+1}^{k^{i}}\left(s_{h+1}^{k^{i}}\right)- \ul{V}_{h+1}^{k^{i}}\left(s_{h+1}^{k^{i}}\right) +2\beta_{i}\right]\\
	= & \alpha_{t}^{0} H+\sum_{i=1}^{t} \alpha_{t}^{i} \delta_{h+1}^{k^i} + 2\sum_{i=1}^{t} \alpha_{t}^{i}\beta_{i}\\
	\leq &  \alpha_{t}^{0} H+\sum_{i=1}^{t} \alpha_{t}^{i} \delta_{h+1}^{k^i} + c \sqrt{AH^4 \iota/t},
	\end{aligned}
	\] 
	for some constant $c$, and the last step is due to Lemma~\ref{lemma:alpha}. Summing over $k$, notice that 
	\[
	\sum_{k=1}^{K} \alpha_{n_{h}^{k}}^{0} H=\sum_{k=1}^{K} H \indicator\left\{n_{h}^{k}=0\right\} \leq HS,
	\]
	because there are at most $SH$ pairs of $(s,h)$ to be visited. Further,
	\[
	\begin{aligned}
	\sum_{k=1}^{K} \sum_{i=1}^{n_{h}^{k}} \alpha_{n_{h}^{k}}^{i} \delta_{h+1}^{k_{h}^{i}\left(s_{h}^{k}\right)} \leq & \sum_{k^{\prime}=1}^{K} \delta_{h+1}^{k^{\prime}} \sum_{i=n_{h}^{k^{\prime}}+1}^{\infty} \alpha_{i}^{n_{h}^{k^{\prime}}}\\
	\leq& \left(1+\frac{1}{H}\right) \sum_{k=1}^{K} \delta_{h+1}^{k},
	\end{aligned}
	\]
	where the first step is by switching the order of summation, and the second uses the fact that $\sum_{t=i}^{\infty} \alpha_{t}^{i}=1+\frac{1}{H} \text { for every } i \geq 1$ from Lemma~\ref{lemma:alpha}. Therefore,
	\begin{align}
	\sum_{k=1}^{K}\delta_h^k \leq& \sum_{k=1}^{K} \alpha_{n_{h}^{k}}^{0} H + \sum_{k=1}^{K} \sum_{i=1}^{n_{h}^{k}} \alpha_{n_{h}^{k}}^{i} \delta_{h+1}^{k_{h}^{i}\left(s_{h}^{k}\right)} + \sum_{k=1}^{K}c \sqrt{AH^4 \iota/n_h^k}\nonumber\\
	\leq & HS +  \left(1+\frac{1}{H}\right) \sum_{k=1}^{K} \delta_{h+1}^{k} + \sum_{k=1}^{K}c \sqrt{AH^4 \iota/n_h^k}\label{eqn:delta}. 
	\end{align}
	Applying this formula recursively for $h = H, H-1, \dots, 1$ yields
	\[
	\begin{aligned}
	\sum_{k=1}^{K}\delta_1^k \leq eSH^2 + ec \sum_{h=1}^H \sum_{k=1}^{K} \sqrt{AH^4 \iota/n_h^k},
	\end{aligned}
	\]
	where we used the fact that $(1+\frac{1}{H})^H \leq e$. Finally, for any $h\in [H]$, 
	\[
	\sum_{k=1}^{K} \sqrt{AH^4 \iota/n_h^k} = \sum_{s\in\mc{S}} \sum_{n=1}^{N_h^K(s)} \sqrt{AH^4 \iota/n} \leq O(\sqrt{H^4SAK\iota}),
	\]
	where the last step holds because $\sum_{s\in\mc{S}} N_h^K(s) = K$, and the LHS is maximized when $N_h^K(s) = K/S$ for all $s\in\mc{S}$. Summarizing the results above leads to the desired bound
	\[
	\sum_{k=1}^K \l V_{k,1}^{\star, \bar{\nu}}(s_1) -  V_{k,1}^{\bar{\mu}, \bar{\nu}}(s_1)\r \leq \sum_{k=1}^{K}\delta_1^k \leq O(\sqrt{H^6SAK\iota}).
	\]
\end{proof}

From the relationship between $(\bar{\mu},\bar{\nu})$ and $(\bar{\mu}_1^k,\bar{\nu}_1^k)$, and that $V_1^{\bar{\mu},\bar{\nu}}(s_1) = \frac{1}{K}\sum_{k=1}^K V_1^{\bar{\mu}_1^k,\bar{\nu}_1^k}(s_1)$, we can immediately conclude from Theorem~\ref{thm:main} that the correlated policy $(\bar{\mu},\bar{\nu})$ constitutes an approximate CCE.

\begin{corollary}\label{corollary}
	(Sample complexity of V-learning OMD). For any $p\in (0,1]$, set $\iota = \log(2S\max\{A,B\}T/p)$, and let the two agents run Algorithm~\ref{alg:qomd} for $K$ episodes with $K= \Omega(H^6 S \max\{A,B\}\iota/\epsilon^2)$. Then, with probability at least $1-p$, the two agents can obtain an $\epsilon$-approximate coarse correlated equilibrium using a common random seed. 
\end{corollary}

Finally, to obtain a sample complexity lower bound for the problem, one simple way is to consider a Markov game instance where either $\mc{A}$ or $\mc{B}$ is a singleton, i.e., $A = 1$ or $B = 1$. In this case, there is no need to correlate the actions of the agents, and hence a CCE in such a game reduces to a NE. In addition, learning a NE against an opponent with a fixed policy is equivalent to learning an optimal policy in a fixed environment. Hence, we have reduced the problem of learning a CEE in a Markov game to a single-agent RL problem either for agent 2 or for agent 1. Applying the regret lower bound of single-agent RL yields the following result for RL in Markov games.
\begin{corollary}\label{corollary2} (Corollary of Theorem 5 in~\citet{jaksch2010near}). 
	For any algorithm, the sample complexity on achieving an $\epsilon$-approximate CCE in two-player general-sum Markov games is at least $\Omega(H^3 S\max\{A,B\}/\epsilon^2)$. 
\end{corollary}

Comparing Corollaries \ref{corollary} and \ref{corollary2}, we see that the sample complexity of Algorithm~\ref{alg:qomd} matches the information-theoretical lower bound in terms of the dependences on $S,A,B$ and $\epsilon$, leaving a gap only in the dependence of $H$.  Notably, the tight dependence on $\max\{A,B\}$ is a natural benefit from decentralized learning, which would not have been achieved by centralized approaches.

\section{Adversarial Bandits with Weighted Regret}\label{sec:omd}
In this section, we close the gap in the proof of Lemma~\ref{lemma:upperbound} by formally presenting a bandit regret bound that we used in \eqref{eqn:a1}. Specifically, we consider an adversarial multi-armed bandit problem, and propose an online mirror descent based algorithm for this problem, which also serves as an important subroutine in Algorithm~\ref{alg:qomd}. Our OMD algorithm achieves an anytime high probability bound with respect to a weighted definition of regret. Such a result complements the Follow the Regularized Leader based algorithm in~\citet{bai2020near} and might be of independent interest.

Specifically, we consider an $A$-armed bandit problem, i.e., the action space is $\mc{A} = \{1,2,\dots,A\}$. The arms are associated with an adversarial sequence of loss vectors $(l_t)_{t=1}^T$, where $l_t \in [0,1]^A$. The bandit proceeds for $T$ rounds. At each round $t$, the player specifies a distribution $\theta_t\in\Delta(\mc{A})$ over the actions, and takes an action $a_t$ sampled from this distribution. We consider bandit feedback, where the player only observes the loss associated with the chosen action $l_t(a_t)$. The player's objective is to minimize the weighted regret with respect to the best fixed policy in hindsight for any time step $t\in[T]$:
\[
\text{Reg}_t(\theta^\star)  \defeq \sum_{i=1}^{t} w_{i} \mathbb{E}_{a \sim \theta^{\star}}\left[l_{i}(a_i)-l_{i}\left(a\right) \mid \mathcal{F}_{i}\right]=\sum_{i=1}^{t} w_{i}\left\langle\theta_{i}-\theta^{*}, l_{i}\right\rangle, 
\]
where  $\theta^\star \in \Delta(\mc{A})$ is an arbitrary but fixed policy, $0 \leq w_i \leq 1$ is the weight of the regret for round $i$, and $\mc{F}_{i}$ is the $\sigma$-algebra generated by the events up to and including round $i-1$. We can check that such a problem formulation indeed captures the adversarial bandit subroutine with weighted regret used in the analysis of Lemma~\ref{lemma:upperbound}.

\begin{algorithm*}[!t]
	\textbf{Input:} The weight of the regret $w_t\in[0,1]$ for each round $t$.
	
	\textbf{Initialize:} ${\theta}_{1}\gets \mathbf{1}/A \defeq (\frac{1}{A}, \dots, \frac{1}{A})$. 
	
	\For{$t\gets 1$ to $T$}
	{
		Take action $a_t\sim \theta_t$, and observe loss $\tilde{l}_t(a_t)$\;
		$\hat{l}_{t}(a) \leftarrow \tilde{l}_{t}(a) \indicator\left\{a_{t}=a\right\} /\left(\theta_{t}(a)+\gamma_{t}\right)$ for all $a \in \mathcal{A}$, where $\gamma_{t}=\sqrt{\frac{\log A}{A t}}$\;
		$\tilde{\theta}_{t+1} \gets \arg\min_{\theta \in \mc{D}} \left\{ \eta_t \inner{\theta,  \hat{l}_t} + D_F(\theta, \theta_t)\right\}$, where $\eta_t \gets \sqrt{\frac{\log A}{At}}$\;
		$\theta_{t+1}'\gets \arg\min_{\theta\in \Delta(\mc{A})} D_F(\theta, \tilde{\theta}_{t+1})$\;
		$\theta_{t+1}\gets \beta_t \theta_{t+1}' + (1-\beta_t) \theta_1$, where $\beta_t \gets \frac{\eta_{t+1}w_t}{\eta_t w_{t+1}}$\;
	}
	\caption{Stabilized Online Mirror Descent with Weighted Regret}\label{alg:omd}
\end{algorithm*}

We present our OMD-based algorithm in Algorithm~\ref{alg:omd}. We again use the unnormalized negentropy regularizer $F(\theta) = \sum_{a=1}^A (\theta(a) \log(\theta(a)) - \theta(a))$ with domain $\mc{D} = \text{dom}(F)$. Direct calculation shows that the Bregman divergence with respect to $F$ is 
$$
D_F(u,v)= F(u) - F(v) - \inner{u-v, \grad F(v)} = \sum_{a=1}^A u(a) \log(u(a)/v(a)), 
$$
which coincides with the Kullback–Leibler divergence when $u$ and $v$ are defined on the simplex. 

The structure of Algorithm~\ref{alg:omd} essentially follows the well-developed OMD framework, but with the following two critical refinements in order to achieve an anytime high-probability regret bound: First, to establish high-probability regret guarantees, we use an \emph{implicit exploration} technique \citep{neu2015explore}, and deliberately maintain a biased estimate of the true losses as
\[
\hat{l}_{t}(a) \leftarrow \frac{\tilde{l}_{t}(a)}{\theta_{t}(a)+\gamma_{t}} \indicator\left\{a_{t}=a\right\}.
\]
We can show (in Lemma~\ref{lemma:highprob} below) that with an appropriately chosen $\gamma_t > 0$, loss estimates of this form constitute a lower confidence bound of the true losses, and hence are critical in establishing high-probability regret guarantees of the bandit problem. Second, to achieve an anytime regret bound, we use a stabilization technique~\citep{fang2020online} by replacing the policy at each step with a convex combination of this policy and the initial policy (Line 8). It has been shown in \citet{orabona2018scale} that standard OMD with an $\eta_t \propto \sqrt{1/t}$ learning rate can incur linear regret when the Bregman divergence is unbounded. To resolve this unboundedness, the stabilization technique mixes a small fraction of $\theta_1$ into each iterate $\theta_t$. In this sense, every iterate $\theta_t$ remains somewhat close (with respect to the Bregman divergence) to the point $\theta_1$. Since the distance between $\theta_1$ and any other point in $\Delta(\mc{A})$ is small (due to our initialization of $\theta_1$), we know that each iterate $\theta_t$ is also not too far from all the other points in $\Delta(\mc{A})$. This hence ensures that the Bregman divergences involved with the iterates are always bounded. In the original stabilization technique \citep{fang2020online}, a $(1-\frac{\eta_{t+1}}{\eta_t})$-fraction of $\theta_1$ is mixed into each iterate $\theta_t$; while in our algorithm, this fraction is set to $1-\frac{\eta_{t+1}w_t}{\eta_t w_{t+1}}$ because we need to additionally address the weighted regret. The following theorem presents the regret guarantee of Algorithm~\ref{alg:omd}.

\begin{theorem}\label{thm:omd_regret}
	For any $p\in(0,1]$, let $\iota = \log(p/AT)$. For any $t\in[T]$, suppose  $\eta_i \leq 2\gamma_i$, $0\leq w_i\leq 1$, $\beta_i\in (0,1],\forall   i \in [t]$, and $\gamma_i$ is non-increasing in $i$. Then, with probability at least $1-3p$, the weighted regret of Algorithm~\ref{alg:omd} is upper bounded by:
	\[
	\text{Reg}_t(\theta^\star) \leq 2 \max _{i \leq t} w_{i} \sqrt{A t\iota}+\frac{3 \sqrt{A \iota}}{2} \sum_{i=1}^{t} \frac{w_{i}}{\sqrt{i}}+\frac{1}{2} \max _{i \leq t} w_{i} \iota+\sqrt{2 \iota \sum_{i=1}^{t} w_{i}^{2}}.
	\] 
\end{theorem}
We can verify that our choices of the parameter values in Algorithm~\ref{alg:qomd} indeed satisfy the requirements in Theorem~\ref{thm:omd_regret}, that is, $\eta_i \leq 2\gamma_i$, $0\leq w_i\leq 1$, $\beta_i\in (0,1],\forall   i \in [t]$, and $\gamma_i$ is non-increasing in $i$. Therefore, the regret bound in Theorem~\ref{thm:omd_regret} can be applied to the proof of Lemma~\ref{lemma:upperbound}. The only caution is that in this section we have assumed for simplicity that the loss function is bounded in $[0,1]$, while the actual losses in Section~\ref{sec:analysis} are bounded in $[0,H]$. Hence, multiplying the regret bound in Theorem~\ref{thm:omd_regret} by a factor of $H$ leads to the result in \eqref{eqn:a1}. 

A final remark is that Algorithm~\ref{alg:omd} assumes that the weights of the regret $w_i$ for $1\leq i\leq t$ are given a priori; but when Algorithm~\ref{alg:omd} is utilized as a subroutine in Algorithm~\ref{alg:qomd}, the weight $w_i$ at round $i$ actually corresponds to $\alpha_t^i$, which cannot be pre-computed when $t$ is not given. To address this subtlety, we specifically design Algorithm~\ref{alg:omd} in a way such that the weights $w_i$ influence the algorithm only through $\beta_i \gets \frac{\eta_{i+1}w_i}{\eta_i w_{i+1}}$ (Line 8). By the definition of $\alpha_t^i$, we see that $\frac{w_i}{w_{i+1}} = \frac{\alpha_t^i}{\alpha_t^{i+1}} = \frac{\alpha_i (1-\alpha_{i+1})}{\alpha_{i+1}}$ can be calculated even when the value of $t$ is unknown. In this way, Algorithm~\ref{alg:qomd} has bypassed the subtlety that the weights of the regret should be given beforehand as required in Algorithm~\ref{alg:omd}. 

\begin{proof} (of Theorem~\ref{thm:omd_regret}). 
	The weighted regret $\text{Reg}_t(\theta^\star)$ can be decomposed into three terms:
	\[
	\begin{aligned}
	\text{Reg}_t(\theta^\star)
	=&\sum_{i=1}^{t} w_{i}\left\langle\theta_{i}-\theta^{\star}, l_{i}\right\rangle \\
	=& \underbrace{\sum_{i=1}^{t} w_{i}\left\langle\theta_{i}-\theta^{\star}, \hat{l}_{i}\right\rangle}_{\tiny \circled{A}}
	+\underbrace{\sum_{i=1}^{t} w_{i}\left\langle\theta_{i}, l_i-\hat{l}_{i}\right\rangle}_{\tiny \circled{B}}
	+\underbrace{\sum_{i=1}^{t} w_{i}\left\langle\theta^{\star}, \hat{l}_i - l_{i}\right\rangle}_{\tiny \circled{C}}. 
	\end{aligned}
	\]
	We bound each of the three terms ${\tiny \circled{A}}, {\tiny \circled{B}}$ and ${\tiny \circled{C}}$ in Lemmas \ref{lemma:19}, \ref{lemma:20}, and \ref{lemma:21} below, respectively. By setting $\eta_t = \gamma_t = \sqrt{\frac{\log A}{At}}$, we can verify that the conditions in Lemma~\ref{lemma:19} and Lemma~\ref{lemma:21} are satisfied. We specifically define $\eta_{t+1}=\eta_t$ and $w_{t+1} = w_t$. One can verify that these two parameters influence the algorithm only through $\beta_t$, and the results stated in the lemmas still hold. Plugging back the results and taking a union bound, it holds with probability at least $1-3p$:
	\[
	\begin{aligned}
	\text{Reg}_t(\theta^\star)\leq& \frac{w_{t+1}\log A}{\eta_{t+1} } + \frac{A}{2}\sum_{i=1}^t \eta_i w_i  + \frac{1}{2}\max_{i \leq t} w_i \iota \\
	&+  A \sum_{i=1}^{t} \gamma_{i} w_{i}+\sqrt{2 \iota \sum_{i=1}^{t} w_{i}^{2}} + \max _{i \leq t} w_{i} \iota / \gamma_{t}\\
	=& 2 \max _{i \leq t} w_{i} \sqrt{A t \iota}+\frac{3 \sqrt{A \iota}}{2} \sum_{i=1}^{t} \frac{w_{i}}{\sqrt{i}}+\frac{1}{2} \max _{i \leq t} w_{i} \iota+\sqrt{2 \iota \sum_{i=1}^{t} w_{i}^{2}}.
	\end{aligned}
	\] 
	This completes the proof of Theorem~\ref{thm:omd_regret}. 
\end{proof}

In the rest of this section, we present some lemmas that were used in the proof of Theorem~\ref{thm:omd_regret}. Before we start, we first recall the following two properties of the Bregman divergence that will be useful in our analysis. 

\begin{lemma}\label{lemma:pythagorean}
	(Pythagorean theorem for Bregman divergence, Lemma 4.1 in \cite{bubeck2015convex}). Let $\mc{X} \subseteq \rr^n$ be a convex set, $y\in\rr^n$, and $z = \arg\min_{u\in\mc{X}}D_F(u,y) $. Then, for any $x\in\mc{X}$, 
	\[
	D_F(x, y) - D_F(z, y) \geq D_F(x, z).  
	\]
\end{lemma}

\begin{lemma} \label{lemma:convexity}
	(Convexity). Let $\mc{X}\subseteq \rr^n$ be the $(n-1)$-dimensional simplex, and let $F$ be the unnormalized negentropy regularizer. For any $x,y\in\mc{X}$, the mapping $D_F(x,\cdot)$ is convex on $\mc{X}$. 
\end{lemma}

We start with the following technical result given in \cite{bai2020near}, which was in turn adapted from Lemma 1 in \cite{neu2015explore}. This lemma allows us to construct high probability regret bounds for Algorithm~\ref{alg:omd}, rather than only regret bounds in expectation. 
\begin{lemma}\label{lemma:highprob}
	(Lemma 18 in \cite{bai2020near}) For any sequence of coefficients $c_{1}, c_{2}, \ldots, c_{t}$ s.t. $c_{i} \in\left[0,2 \gamma_{i}\right]^A$ is $\mathcal{F}_{i}$ -measurable, we have with probability at least $1-p / AT$,
	\[
	\sum_{i=1}^{t} w_{i}\left\langle c_{i}, \hat{l}_{i}-l_{i}\right\rangle \leq \max _{i \leq t} w_{i} \iota.
	\]
\end{lemma}

\begin{lemma} \label{lemma:omd_regret}
	Suppose $\beta_i\in (0,1],\forall   i \in [t]$. For any fixed policy $\theta\in\Delta(\mc{A})$ and for any time step $t\in[T]$, the weighted regret of Algorithm~\ref{alg:omd} with respect to $\theta$ can be bounded by:
	\[
	\sum_{i=1}^{t} w_{i}\left\langle\theta_{i}-\theta, \hat{l}_{i}\right\rangle \leq \frac{w_{t+1}D_F(\theta,\theta_1)}{\eta_{t+1} }  + \sum_{i=1}^t  \frac{w_i D_F(\theta_{i},\tilde{\theta}_{i+1}) }{\eta_i }. 
	\]
\end{lemma}
\begin{proof}
	Since $\tilde{\theta}_{i+1} = \arg\min_{\theta \in \mc{D}} \left\{ \eta_i \inner{\theta,  \hat{l}_i} + D_F(\theta, \theta_i)\right\}$, the first-order optimality condition implies that 
	\[
	\eta_i  \hat{l}_i + \grad F(\tilde{\theta}_{i+1}) - \grad F (\theta_i) = 0.
	\] 
	Reordering and using the definition of the Bregman divergence,
	\begin{align}
	\inner{\theta_{i} - \theta, \hat{l}_i} =& \frac{1}{\eta_i } \inner{\theta_{i}-\theta,\grad F(\theta_i)-\grad F(\theta_{i+1})}\nonumber \\
	=&  \frac{1}{\eta_i }(D_F(\theta,\theta_i)-D_F(\theta,\tilde{\theta}_{i+1}) +D_F(\theta_{i},\tilde{\theta}_{i+1})). \label{eqn:a7}
	\end{align}
	By the Pythagorean theorem for Bregman divergence (Lemma~\ref{lemma:pythagorean}), 
	\[
	\begin{aligned}
	&\beta_i (D_F(\theta,\tilde{\theta}_{i+1}) - D_F(\theta_{i+1}', \tilde{\theta}_{i+1})) + (1-\beta_i) D_F(\theta, \theta_1)\\
	\geq& \beta_i D_F(\theta,\theta_{i+1}') + (1-\beta_i) D(\theta,\theta_1)\\
	\geq &D_F(\theta,\theta_{i+1}),
	\end{aligned}
	\]
	where the second step is by the convexity of $D_F(\theta,\cdot)$ (Lemma~\ref{lemma:convexity}) and the fact that $\theta_{i+1}= \beta_t \theta_{i+1}' + (1-\beta_i) \theta_1$. Rearranging the terms yields
	\[
	D_F(\theta,\tilde{\theta}_{i+1}) \geq \frac{1}{\beta_i} D_F(\theta,\theta_{i+1}) - \frac{1-\beta_i}{\beta_i} D_F(\theta,\theta_1) + D_F(\theta_{i+1}',\tilde{\theta}_{i+1}).
	\]
	Plugging this into \eqref{eqn:a7} and recalling the definition that $\beta_i = \eta_{i+1}/\eta_i$, we obtain
	\[
	\begin{aligned}
	&w_i \inner{\theta_{i} - \theta, \hat{l}_i} =  \frac{w_i}{\eta_i }(D_F(\theta,\theta_i)-D_F(\theta,\tilde{\theta}_{i+1}) +D_F(\theta_{i},\tilde{\theta}_{i+1}))\\
	\leq & \frac{w_i}{\eta_i }(D_F(\theta,\theta_i)- \frac{1}{\beta_i} D_F(\theta,\theta_{i+1}) + \frac{1-\beta_i}{\beta_i} D_F(\theta,\theta_1) - D_F(\theta_{i+1}' , \tilde{\theta}_{i+1}) +D_F(\theta_{i},\tilde{\theta}_{i+1}))\\
	=& \frac{w_i D_F(\theta,\theta_i)}{\eta_i } - \frac{w_{i+1}D_F(\theta,\theta_{i+1})}{\eta_{i+1}} +	\l \frac{w_{i+1}}{\eta_{i+1}}-\frac{w_i}{\eta_i } \r  D_F(\theta,\theta_1) \\
	&\qquad \qquad \quad\  - \frac{w_iD_F(\theta_{i+1}' , \tilde{\theta}_{i+1})}{\eta_i } + \frac{w_i D_F(\theta_{i},\tilde{\theta}_{i+1})}{\eta_i }. 
	\end{aligned}
	\]
	Summing over $i$ and telescoping leads to
	\[
	\begin{aligned}
	&\sum_{i=1}^{t} w_{i}\left\langle\theta_{i}-\theta, \hat{l}_{i}\right\rangle\\
	\leq &\frac{w_1 D_F(\theta,\theta_1)}{\eta_1 } + \sum_{i=1}^t \l \frac{w_{i+1}}{\eta_{i+1}}-\frac{w_i}{\eta_i } \r  D_F(\theta,\theta_1) + \sum_{i=1}^t  \frac{w_i(D_F(\theta_{i},\tilde{\theta}_{i+1}) - D_F(\theta_{i+1}' , \tilde{\theta}_{i+1}))}{\eta_i }\\
	=&\frac{w_{t+1}D_F(\theta,\theta_1)}{\eta_{t+1} }  + \sum_{i=1}^t  \frac{w_i D_F(\theta_{i},\tilde{\theta}_{i+1}) }{\eta_i },
	\end{aligned}
	\]
	where in the last step we used the fact that $D_F(\theta_{i+1}' , \tilde{\theta}_{i+1}) \geq 0$ (by the convexity of $F$). 
\end{proof}

\begin{lemma}\label{lemma:19}
	If $\eta_i \leq 2\gamma_i$ and $0\leq w_i\leq 1$ for all $i\leq t$, it holds with probability at least $1-p$ that
	\[
	\sum_{i=1}^{t} w_{i}\left\langle\theta_{i}-\theta^\star, \hat{l}_{i}\right\rangle \leq \frac{w_{t+1}\log A}{\eta_{t+1} } + \frac{A}{2}\sum_{i=1}^t \eta_i w_i  + \frac{1}{2}\max_{i \leq t} w_i \iota. 
	\]
\end{lemma}
\begin{proof}
	Our proof relies on the following regret bound of OMD given in Lemma~\ref{lemma:omd_regret}: For any $\theta \in \Delta(\mc{A})$ and any $t\in[T]$, 
	\begin{equation}\label{eqn:omd_regret}
	\sum_{i=1}^{t} w_{i}\left\langle\theta_{i}-\theta, \hat{l}_{i}\right\rangle \leq  \frac{w_{t+1}D_F(\theta,\theta_1)}{\eta_{t+1} }  + \sum_{i=1}^t  \frac{w_i D_F(\theta_{i},\tilde{\theta}_{i+1}) }{\eta_i }. 
	\end{equation}
	
	Since $\tilde{\theta}_{i+1} = \arg\min_{\theta \in \mc{D}} \left\{ \eta_i \inner{\theta, \hat{l}_i} + D_F(\theta, \theta_i)\right\}$, the minimum is achieved when $\eta_i  \hat{l}_i + \grad F(\tilde{\theta}_{i+1}) - \grad F (\theta_i) = 0$. Direct calculation shows that $\tilde{\theta}_{i+1}(a) = \theta_i(a) \exp(-\eta_i  \hat{l}_i(a))$ for all $a\in\mc{A}$. Hence,
	\[
	\begin{aligned}
	D_F(\theta_i,\tilde{\theta}_{i+1}) =& \sum_{a=1}^A \theta_i(a) \log\l \frac{\theta_i(a)}{\tilde{\theta}_{i+1}(a)} \r - \sum_{a=1}^A \theta_i(a) +\sum_{a=1}^A \tilde{\theta}_{i+1}(a)\\
	=&\sum_{a=1}^A \theta_i(a) \l \eta_i  \hat{l}_i(a) - 1 + \exp(-\eta_i  \hat{l}_i(a)) \r\\
	\leq &\frac{\eta_i^2}{2} \sum_{a=1}^A \theta_i(a) \hat{l}_i(a)^2,
	\end{aligned}
	\]
	where the last step holds because $\exp(x) \leq 1 + x + x^2/2$ for $x\leq 0$. Plugging this back to Equation \eqref{eqn:omd_regret}, we have that
	\begin{align}
	\sum_{i=1}^{t} w_{i}\left\langle\theta_{i}-\theta, \hat{l}_{i}\right\rangle \leq&\frac{w_{t+1}D_F(\theta,\theta_1)}{\eta_{t+1} } + \frac{1}{2}\sum_{i=1}^t \sum_{a=1}^A \eta_iw_i \theta_i(a)  \hat{l}_i(a)^2 \nonumber\\
	\leq & \frac{w_{t+1}D_F(\theta,\theta_1)}{\eta_{t+1} } + \frac{1}{2}\sum_{i=1}^t \sum_{a=1}^A  \eta_iw_i \hat{l}_i(a)\label{eqn:a6}\\
	\leq & \frac{w_{t+1}D_F(\theta,\theta_1)}{\eta_{t+1} }+ \frac{1}{2}\sum_{i=1}^t \sum_{a=1}^A  \eta_iw_i l_i(a) + \frac{1}{2}\max_{i \leq t} w_i \iota \label{eqn:a5}\\
	\leq & \frac{w_{t+1}\log A}{\eta_{t+1} } + \frac{A}{2}\sum_{i=1}^t \eta_i w_i  + \frac{1}{2}\max_{i \leq t} w_i \iota,\nonumber
	\end{align}
	where \eqref{eqn:a6} holds because $\hat{l}_i(a) \neq 0$ only if $\indicator\{a_i = a\} = 1$, and hence it follows that $\sum_{a=1}^A \theta_i(a)  \hat{l}_i(a)^2 = \theta_i(a_i)\hat{l}_i(a_i)^2 = \theta_i(a_i) \frac{\tilde{l}_i(a_i)}{\theta_i(a_i) + \gamma_i} \hat{l}_i(a_i) \leq \hat{l}_i(a_i) = \sum_{a=1}^A \hat{l}_i(a)$. Step \eqref{eqn:a5} is by applying Lemma~\ref{lemma:highprob}, with $c_i(a) = \eta_i$ for all $1\leq a\leq A$. The last step holds because $D_F(\theta,\theta_1) \leq \log A$ for $\theta_1 = \mbf{1}/A$ and any $\theta\in \Delta(\mc{A})$. 
\end{proof}

\begin{lemma}\label{lemma:20}
	(Lemma 20 in \cite{bai2020near}) With probability at least $1-p$, for any $t \in[T]$, 
	\[
	\sum_{i=1}^{t} w_{i}\left\langle\theta_{i}, l_{i}-\hat{l}_{i}\right\rangle \leq A \sum_{i=1}^{t} \gamma_{i} w_{i}+\sqrt{2 \iota \sum_{i=1}^{t} w_{i}^{2}}.
	\]
\end{lemma}

\begin{lemma}\label{lemma:21}
	(Lemma 21 in \cite{bai2020near}) With probability at least $1-p$, for any $t \in[T]$ and any $\theta^{\star} \in \Delta(\mc{A})$, if $\gamma_{i}$ is non-increasing in $i$, then
	\[
	\sum_{i=1}^{t} w_{i}\left\langle\theta^{*}, \hat{l}_{i}-l_{i}\right\rangle \leq \max _{i \leq t} w_{i} \iota / \gamma_{t}.
	\]
\end{lemma}

\section{Concluding Remarks}\label{sec:conclusions}

In this paper, we have considered multi-agent reinforcement learning with efficient exploration in general-sum Markov games. We have proposed the V-learning OMD algorithm that provably finds an $\epsilon$-approximate coarse correlated equilibrium in at most $\widetilde{O}( H^6S A /\epsilon^2)$ episodes. As a useful side result, we have introduced an anytime online mirror descent algorithm with a dynamic learning rate and a high-probability regret bound. Our algorithm is decentralized and can readily scale up to an arbitrary number of agents without suffering from the exponential dependence.

\bibliographystyle{abbrvnat}
\bibliography{ref}

\begin{thebibliography}{51}
\providecommand{\natexlab}[1]{#1}
\providecommand{\url}[1]{\texttt{#1}}
\expandafter\ifx\csname urlstyle\endcsname\relax
  \providecommand{\doi}[1]{doi: #1}\else
  \providecommand{\doi}{doi: \begingroup \urlstyle{rm}\Url}\fi

\bibitem[Arslan and Y{\"u}ksel(2016)]{arslan2016decentralized}
G.~Arslan and S.~Y{\"u}ksel.
\newblock Decentralized {Q}-learning for stochastic teams and games.
\newblock \emph{IEEE Transactions on Automatic Control}, 62\penalty0
  (4):\penalty0 1545--1558, 2016.

\bibitem[Auer et~al.(1995)Auer, Cesa-Bianchi, Freund, and
  Schapire]{auer1995gambling}
P.~Auer, N.~Cesa-Bianchi, Y.~Freund, and R.~E. Schapire.
\newblock Gambling in a rigged casino: {T}he adversarial multi-armed bandit
  problem.
\newblock In \emph{Foundations of Computer Science}, pages 322--331. IEEE,
  1995.

\bibitem[Aumann(1987)]{aumann1987correlated}
R.~J. Aumann.
\newblock Correlated equilibrium as an expression of {B}ayesian rationality.
\newblock \emph{Econometrica: Journal of the Econometric Society}, pages 1--18,
  1987.

\bibitem[Bai and Jin(2020)]{bai2020provable}
Y.~Bai and C.~Jin.
\newblock Provable self-play algorithms for competitive reinforcement learning.
\newblock In \emph{International Conference on Machine Learning}, pages
  551--560, 2020.

\bibitem[Bai et~al.(2020)Bai, Jin, and Yu]{bai2020near}
Y.~Bai, C.~Jin, and T.~Yu.
\newblock Near-optimal reinforcement learning with self-play.
\newblock \emph{Advances in Neural Information Processing Systems}, 33, 2020.

\bibitem[Bernstein et~al.(2002)Bernstein, Givan, Immerman, and
  Zilberstein]{bernstein2002complexity}
D.~S. Bernstein, R.~Givan, N.~Immerman, and S.~Zilberstein.
\newblock The complexity of decentralized control of {M}arkov decision
  processes.
\newblock \emph{Mathematics of Operations Research}, 27\penalty0 (4):\penalty0
  819--840, 2002.

\bibitem[Brown and Sandholm(2018)]{brown2018superhuman}
N.~Brown and T.~Sandholm.
\newblock Superhuman {AI} for heads-up no-limit poker: {L}ibratus beats top
  professionals.
\newblock \emph{Science}, 359\penalty0 (6374):\penalty0 418--424, 2018.

\bibitem[Bubeck et~al.(2015)]{bubeck2015convex}
S.~Bubeck et~al.
\newblock Convex optimization: {A}lgorithms and complexity.
\newblock \emph{Foundations and Trends{\textregistered} in Machine Learning},
  8\penalty0 (3-4):\penalty0 231--357, 2015.

\bibitem[Cesa-Bianchi and Lugosi(2006)]{cesa2006prediction}
N.~Cesa-Bianchi and G.~Lugosi.
\newblock \emph{Prediction, learning, and games}.
\newblock Cambridge University Press, 2006.

\bibitem[Chen et~al.(2009)Chen, Deng, and Teng]{chen2009settling}
X.~Chen, X.~Deng, and S.-H. Teng.
\newblock Settling the complexity of computing two-player {N}ash equilibria.
\newblock \emph{Journal of the ACM}, 56\penalty0 (3):\penalty0 1--57, 2009.

\bibitem[Claus and Boutilier(1998)]{claus1998dynamics}
C.~Claus and C.~Boutilier.
\newblock The dynamics of reinforcement learning in cooperative multiagent
  systems.
\newblock \emph{AAAI Conference on Artificial Intelligence}, 1998\penalty0
  (746-752):\penalty0 2, 1998.

\bibitem[Daskalakis et~al.(2009)Daskalakis, Goldberg, and
  Papadimitriou]{daskalakis2009complexity}
C.~Daskalakis, P.~W. Goldberg, and C.~H. Papadimitriou.
\newblock The complexity of computing a {N}ash equilibrium.
\newblock \emph{SIAM Journal on Computing}, 39\penalty0 (1):\penalty0 195--259,
  2009.

\bibitem[Daskalakis et~al.(2020)Daskalakis, Foster, and
  Golowich]{daskalakis2020independent}
C.~Daskalakis, D.~J. Foster, and N.~Golowich.
\newblock Independent policy gradient methods for competitive reinforcement
  learning.
\newblock \emph{Advances in Neural Information Processing Systems}, 33, 2020.

\bibitem[Fang et~al.(2020)Fang, Harvey, Portella, and
  Friedlander]{fang2020online}
H.~Fang, N.~Harvey, V.~Portella, and M.~Friedlander.
\newblock Online mirror descent and dual averaging: {K}eeping pace in the
  dynamic case.
\newblock In \emph{International Conference on Machine Learning}, pages
  3008--3017, 2020.

\bibitem[Filar and Vrieze(2012)]{filar2012competitive}
J.~Filar and K.~Vrieze.
\newblock \emph{Competitive {M}arkov decision processes}.
\newblock Springer Science \& Business Media, 2012.

\bibitem[Greenwald and Hall(2003)]{greenwald2003correlated}
A.~Greenwald and K.~Hall.
\newblock Correlated-{Q} learning.
\newblock In \emph{International Conference on Machine Learning}, pages
  242--249, 2003.

\bibitem[Hart and Mas-Colell(2000)]{hart2000simple}
S.~Hart and A.~Mas-Colell.
\newblock A simple adaptive procedure leading to correlated equilibrium.
\newblock \emph{Econometrica}, 68\penalty0 (5):\penalty0 1127--1150, 2000.

\bibitem[Ho(1980)]{ho1980team}
Y.-C. Ho.
\newblock Team decision theory and information structures.
\newblock \emph{Proceedings of the IEEE}, 68\penalty0 (6):\penalty0 644--654,
  1980.

\bibitem[Hu and Wellman(2003)]{hu2003nash}
J.~Hu and M.~P. Wellman.
\newblock Nash {Q}-learning for general-sum stochastic games.
\newblock \emph{Journal of Machine Learning Research}, 4\penalty0
  (Nov):\penalty0 1039--1069, 2003.

\bibitem[Jaksch et~al.(2010)Jaksch, Ortner, and Auer]{jaksch2010near}
T.~Jaksch, R.~Ortner, and P.~Auer.
\newblock Near-optimal regret bounds for reinforcement learning.
\newblock \emph{Journal of Machine Learning Research}, 11\penalty0 (4), 2010.

\bibitem[Jin et~al.(2018)Jin, Allen-Zhu, Bubeck, and Jordan]{jin2018q}
C.~Jin, Z.~Allen-Zhu, S.~Bubeck, and M.~I. Jordan.
\newblock Is {Q}-learning provably efficient?
\newblock In \emph{International Conference on Neural Information Processing
  Systems}, pages 4868--4878, 2018.

\bibitem[Kober et~al.(2013)Kober, Bagnell, and Peters]{kober2013reinforcement}
J.~Kober, J.~A. Bagnell, and J.~Peters.
\newblock Reinforcement learning in robotics: {A} survey.
\newblock \emph{International Journal of Robotics Research}, 32\penalty0
  (11):\penalty0 1238--1274, 2013.

\bibitem[Leslie and Collins(2005)]{leslie2005individual}
D.~S. Leslie and E.~J. Collins.
\newblock Individual {Q}-learning in normal form games.
\newblock \emph{SIAM Journal on Control and Optimization}, 44\penalty0
  (2):\penalty0 495--514, 2005.

\bibitem[Littman(1994)]{littman1994markov}
M.~L. Littman.
\newblock Markov games as a framework for multi-agent reinforcement learning.
\newblock In \emph{Machine Learning}, pages 157--163. 1994.

\bibitem[Littman(2001)]{littman2001friend}
M.~L. Littman.
\newblock Friend-or-{F}oe {Q}-learning in general-sum games.
\newblock In \emph{International Conference on Machine Learning}, pages
  322--328, 2001.

\bibitem[Littman and Szepesv{\'a}ri(1996)]{littman1996generalized}
M.~L. Littman and C.~Szepesv{\'a}ri.
\newblock A generalized reinforcement-learning model: {C}onvergence and
  applications.
\newblock In \emph{International Conference on Machine Learning}, pages
  310--318, 1996.

\bibitem[Liu et~al.(2021)Liu, Yu, Bai, and Jin]{liu2021sharp}
Q.~Liu, T.~Yu, Y.~Bai, and C.~Jin.
\newblock A sharp analysis of model-based reinforcement learning with
  self-play.
\newblock In \emph{International Conference on Machine Learning}, pages
  7001--7010. PMLR, 2021.

\bibitem[Moulin and Vial(1978)]{moulin1978strategically}
H.~Moulin and J.-P. Vial.
\newblock Strategically zero-sum games: {T}he class of games whose completely
  mixed equilibria cannot be improved upon.
\newblock \emph{International Journal of Game Theory}, 7\penalty0
  (3-4):\penalty0 201--221, 1978.

\bibitem[Nayyar et~al.(2013{\natexlab{a}})Nayyar, Gupta, Langbort, and
  Ba{\c{s}}ar]{nayyar2013common}
A.~Nayyar, A.~Gupta, C.~Langbort, and T.~Ba{\c{s}}ar.
\newblock Common information based {M}arkov perfect equilibria for stochastic
  games with asymmetric information: {F}inite games.
\newblock \emph{IEEE Transactions on Automatic Control}, 59\penalty0
  (3):\penalty0 555--570, 2013{\natexlab{a}}.

\bibitem[Nayyar et~al.(2013{\natexlab{b}})Nayyar, Mahajan, and
  Teneketzis]{nayyar2013decentralized}
A.~Nayyar, A.~Mahajan, and D.~Teneketzis.
\newblock Decentralized stochastic control with partial history sharing: {A}
  common information approach.
\newblock \emph{IEEE Transactions on Automatic Control}, 58\penalty0
  (7):\penalty0 1644--1658, 2013{\natexlab{b}}.

\bibitem[Nemirovskij and Yudin(1983)]{nemirovskij1983problem}
A.~S. Nemirovskij and D.~B. Yudin.
\newblock \emph{Problem complexity and method efficiency in optimization}.
\newblock Wiley-Interscience, 1983.

\bibitem[Neu(2015)]{neu2015explore}
G.~Neu.
\newblock Explore no more: {I}mproved high-probability regret bounds for
  non-stochastic bandits.
\newblock \emph{Advances in Neural Information Processing Systems},
  28:\penalty0 3168--3176, 2015.

\bibitem[Orabona and P{\'a}l(2018)]{orabona2018scale}
F.~Orabona and D.~P{\'a}l.
\newblock Scale-free online learning.
\newblock \emph{Theoretical Computer Science}, 716:\penalty0 50--69, 2018.

\bibitem[Papadimitriou and Roughgarden(2008)]{papadimitriou2008computing}
C.~H. Papadimitriou and T.~Roughgarden.
\newblock Computing correlated equilibria in multi-player games.
\newblock \emph{Journal of the ACM}, 55\penalty0 (3):\penalty0 1--29, 2008.

\bibitem[P{\'e}rolat et~al.(2017)P{\'e}rolat, Strub, Piot, and
  Pietquin]{perolat2017learning}
J.~P{\'e}rolat, F.~Strub, B.~Piot, and O.~Pietquin.
\newblock Learning {N}ash equilibrium for general-sum {M}arkov games from batch
  data.
\newblock In \emph{Artificial Intelligence and Statistics}, pages 232--241.
  PMLR, 2017.

\bibitem[Prasad et~al.(2015)Prasad, LA, and Bhatnagar]{prasad2015two}
H.~Prasad, P.~LA, and S.~Bhatnagar.
\newblock Two-timescale algorithms for learning {N}ash equilibria in
  general-sum stochastic games.
\newblock In \emph{International Conference on Autonomous Agents and Multiagent
  Systems}, pages 1371--1379, 2015.

\bibitem[Shalev-Shwartz et~al.(2016)Shalev-Shwartz, Shammah, and
  Shashua]{shalev2016safe}
S.~Shalev-Shwartz, S.~Shammah, and A.~Shashua.
\newblock Safe, multi-agent, reinforcement learning for autonomous driving.
\newblock \emph{arXiv preprint arXiv:1610.03295}, 2016.

\bibitem[Shapley(1953)]{shapley1953stochastic}
L.~S. Shapley.
\newblock Stochastic games.
\newblock \emph{Proceedings of the National Academy of Sciences}, 39\penalty0
  (10):\penalty0 1095--1100, 1953.

\bibitem[Sidford et~al.(2020)Sidford, Wang, Yang, and Ye]{sidford2020solving}
A.~Sidford, M.~Wang, L.~Yang, and Y.~Ye.
\newblock Solving discounted stochastic two-player games with near-optimal time
  and sample complexity.
\newblock In \emph{International Conference on Artificial Intelligence and
  Statistics}, pages 2992--3002. PMLR, 2020.

\bibitem[Silver et~al.(2016)Silver, Huang, Maddison, Guez, Sifre, Van
  Den~Driessche, Schrittwieser, Antonoglou, Panneershelvam, Lanctot,
  et~al.]{silver2016mastering}
D.~Silver, A.~Huang, C.~J. Maddison, A.~Guez, L.~Sifre, G.~Van Den~Driessche,
  J.~Schrittwieser, I.~Antonoglou, V.~Panneershelvam, M.~Lanctot, et~al.
\newblock Mastering the game of {G}o with deep neural networks and tree search.
\newblock \emph{Nature}, 529\penalty0 (7587):\penalty0 484--489, 2016.

\bibitem[Tian et~al.(2021)Tian, Wang, Yu, and Sra]{tian2020provably}
Y.~Tian, Y.~Wang, T.~Yu, and S.~Sra.
\newblock Online learning in unknown {M}arkov games.
\newblock \emph{International Conference on Machine Learning}, 2021.

\bibitem[Vinyals et~al.(2019)Vinyals, Babuschkin, Czarnecki, Mathieu, Dudzik,
  Chung, Choi, Powell, Ewalds, Georgiev, et~al.]{vinyals2019grandmaster}
O.~Vinyals, I.~Babuschkin, W.~M. Czarnecki, M.~Mathieu, A.~Dudzik, J.~Chung,
  D.~H. Choi, R.~Powell, T.~Ewalds, P.~Georgiev, et~al.
\newblock Grandmaster level in {S}tar{C}raft {II} using multi-agent
  reinforcement learning.
\newblock \emph{Nature}, 575\penalty0 (7782):\penalty0 350--354, 2019.

\bibitem[Wang and Sandholm(2002)]{wang2002reinforcement}
X.~Wang and T.~Sandholm.
\newblock Reinforcement learning to play an optimal {N}ash equilibrium in team
  {M}arkov games.
\newblock \emph{Advances in Neural Information Processing Systems},
  15:\penalty0 1603--1610, 2002.

\bibitem[Wei et~al.(2017)Wei, Hong, and Lu]{wei2017online}
C.-Y. Wei, Y.-T. Hong, and C.-J. Lu.
\newblock Online reinforcement learning in stochastic games.
\newblock In \emph{International Conference on Neural Information Processing
  Systems}, pages 4994--5004, 2017.

\bibitem[Wei et~al.(2021)Wei, Lee, Zhang, and Luo]{wei2021last}
C.-Y. Wei, C.-W. Lee, M.~Zhang, and H.~Luo.
\newblock Last-iterate convergence of decentralized optimistic gradient
  descent/ascent in infinite-horizon competitive {M}arkov games.
\newblock \emph{Annual Conference on Learning Theory}, 2021.

\bibitem[Xie et~al.(2020)Xie, Chen, Wang, and Yang]{xie2020learning}
Q.~Xie, Y.~Chen, Z.~Wang, and Z.~Yang.
\newblock Learning zero-sum simultaneous-move {M}arkov games using function
  approximation and correlated equilibrium.
\newblock In \emph{Conference on Learning Theory}, pages 3674--3682, 2020.

\bibitem[Yongacoglu et~al.(2019)Yongacoglu, Arslan, and
  Y{\"u}ksel]{yongacoglu2019learning}
B.~Yongacoglu, G.~Arslan, and S.~Y{\"u}ksel.
\newblock Learning team-optimality for decentralized stochastic control and
  dynamic games.
\newblock \emph{arXiv preprint arXiv:1903.05812}, 2019.

\bibitem[Y{\"u}ksel and Ba{\c{s}}ar(2013)]{yuksel2013stochastic}
S.~Y{\"u}ksel and T.~Ba{\c{s}}ar.
\newblock \emph{Stochastic networked control systems: {S}tabilization and
  optimization under information constraints}.
\newblock Springer Science \& Business Media, 2013.

\bibitem[Zehfroosh and Tanner(2020)]{zehfroosh2020pac}
A.~Zehfroosh and H.~G. Tanner.
\newblock {PAC} reinforcement learning algorithm for general-sum {M}arkov
  games.
\newblock \emph{arXiv preprint arXiv:2009.02605}, 2020.

\bibitem[Zhao et~al.(2021)Zhao, Tian, Lee, and Du]{zhao2021provably}
Y.~Zhao, Y.~Tian, J.~D. Lee, and S.~S. Du.
\newblock Provably efficient policy gradient methods for two-player zero-sum
  {M}arkov games.
\newblock \emph{arXiv preprint arXiv:2102.08903}, 2021.

\bibitem[Zinkevich(2003)]{zinkevich2003online}
M.~Zinkevich.
\newblock Online convex programming and generalized infinitesimal gradient
  ascent.
\newblock In \emph{International Conference on Machine Learning}, pages
  928--936, 2003.

\end{thebibliography}

\end{document}